\documentclass[10pt,letterpaper]{article}

\usepackage{algorithm}
\usepackage{algorithmic}
\usepackage{enumerate}
\usepackage{amsmath, amsfonts, amssymb}
\usepackage{amsthm}
\usepackage{mathrsfs}
\usepackage{bm}
\usepackage{bbm}
\usepackage{graphicx}
\usepackage{alltt}
\usepackage{tikz}
\usepackage{float}
\usepackage{booktabs}
\usepackage{array}
\usepackage{dsfont}
\usepackage{nicefrac}
\usepackage{comment}
\usepackage{anyfontsize}
\usepackage{longtable}
\usepackage{natbib}
\usepackage{thmtools}
\usepackage{thm-restate}
\usepackage{wrapfig}
\usepackage[symbol]{footmisc}
\allowdisplaybreaks

\newcommand{\norm}[1]{\left\lVert#1\right\rVert}
\newcommand{\inprod}[2]{\left\langle#1, #2\right\rangle}

\newcommand{\Par}[1]{\left(#1\right)}
\newcommand{\Brack}[1]{\left[#1\right]}
\newcommand{\Brace}[1]{\left\{#1\right\}}
\newcommand{\R}{\mathbb{R}}
\newcommand{\E}{\mathbb{E}}
\newcommand{\tx}{\tilde{x}}
\newcommand{\ty}{\tilde{y}}
\newcommand{\tz}{\tilde{z}}
\newcommand{\oR}{\overline{R}}

\newcommand{\eps}{\epsilon}
\newcommand{\teps}{\tilde{\epsilon}}
\newcommand{\topt}{t_{\textup{opt}}}
\newcommand{\tear}{\overline{t}_{\textup{opt}}}

\newcommand{\so}{\mathcal{S}_{\text{over}}}
\newcommand{\su}{\mathcal{S}_{\text{under}}}
\newcommand{\id}{\mathbf{I}}
\newcommand{\tr}{\textup{tr}}
\newcommand{\indi}{\mathbbm{1}}

\newcommand*\lrp[1]{\left(#1\right)}
\newcommand*\lrn[1]{\left\|#1\right\|}
\newcommand*\ind{\mathds{1}}

\newtheorem{lemma}{Lemma}

\newtheorem{remark}{Remark}

\definecolor{blue}{rgb}{0.2, 0.3, 0.7}

\definecolor{liyao}{rgb}{0.5, 0.5, 0.5}

\usepackage[normalem]{ulem} \usepackage{hyperref}
\usepackage{color}
\definecolor{ForestGreen}{rgb}{0.1333,0.5451,0.1333}
\hypersetup{colorlinks,
	linkcolor=ForestGreen,
	citecolor=ForestGreen,
	urlcolor=black,
	linktocpage,
	plainpages=false}
	
\title{On Optimal Early Stopping:  Over-informative versus
Under-informative Parametrization}

\begin{document}

\date{}
\author{
			Ruoqi Shen\thanks{University of Washington, {\tt shenr3@cs.washington.edu}}
			\and
			Liyao Gao\thanks{University of Washington, {\tt marsgao@uw.edu}}
			\and
			Yi-An Ma\thanks{University of California San Diego, {\tt yianma@ucsd.edu}}
		}
\maketitle

\begin{abstract}
	Early stopping is a simple and widely used method to prevent over-training neural networks. We develop theoretical results to reveal the relationship between the optimal early stopping time and model dimension as well as sample size of the dataset for certain linear models. Our results demonstrate two very different behaviors when the model dimension exceeds the number of features versus the opposite scenario. While most previous works on linear models focus on the latter setting, we observe that the dimension of the model often exceeds the number of features arising from data in common deep learning tasks and propose a model to study this setting. We demonstrate experimentally that our theoretical results on optimal early stopping time corresponds to the training process of deep neural networks.
    \end{abstract}

\section{Introduction}
Generalization, accuracy, and computation are three of the major aspects in deploying large scale machine learning models.
They are often brought together into an optimization framework, where the first two aspects concern stationary solutions and the last aspect concerns convergence properties of the algorithms.
As a result, most previous works have analyzed them separately, seeking to strike a balance between generalization and accuracy first and then understand computational complexity \citep{zhang2002covering}.
These approaches often design a regularizer in the form of a penalty term added to the objective function \citep[c.f.,][and references therein]{nakkiran2020optimal}.
While this method is effective in many settings, it can prove challenging to determine the regularization that guarantees appealing behaviors.
Adding to the complication is the fact that they are often designed to avoid overfitting assuming that the optimization algorithm converges. 
It is not obvious how the behavior will change if the algorithm stops before the convergence of the optimization procedure.

Recent success of deep learning models combined with gradient based optimization has prompted increasingly many works to focus on blending the computation aspect into the picture of generalization--accuracy tradeoff \citep{du2018algorithmic}.
Such approaches often leverage implicit properties of the optimization algorithms in conjunction of the model and the data structures and are thus referred to as implicit regularization methods.
In training of deep neural networks, techniques such as stochastic gradient descent, batch normalization \citep{ioffe2015batch}, and dropouts \citep{srivastava2014dropout} are widely used to grant generalization properties implicitly. 

In this paper, we study regularization via early stopping \citep{morgan1989generalization, zhang2005boosting, yao2007early}. 
While in practice, early stopping is often forced by computational budget constraints, we focus on its efficacy in avoiding overfitting towards the training dataset. 
In the presence of label noise, in particular, we study the common strategy of concluding the training process at the optimal early stopping time, which achieves the lowest test risk before it increases again. 
One crucial question is how this optimal early stopping time relates to the model size and the sample size of data. 
Answering this question not only provides guidance for the model training process in practice, but also contributes to the understanding of the generalization property of different models.

\begin{figure*}
\begin{minipage}{0.65\textwidth}

  \includegraphics[width=1\linewidth]{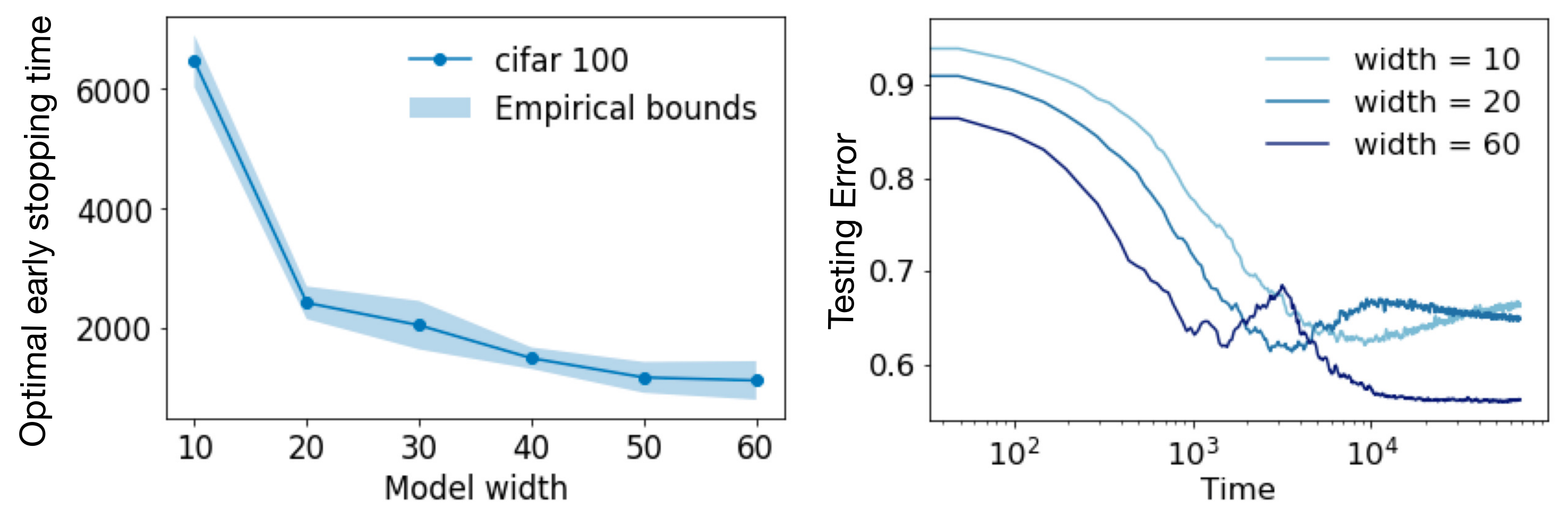}
  \caption{Left: Optimal early stopping time vs. model width $d$ for CIFAR-100 with 20\% label noise trained using ResNet networks with convolutional layers of widths $[d, 2d, 4d, 8d]$. Right: Testing error vs. time for CIFAR-100 with 20\% label noise trained using ResNet networks using different model widths.}
    \label{fig:intro_fig}
\end{minipage}
\hfill
\begin{minipage}{.32\textwidth}
  
  \includegraphics[width=1\linewidth]{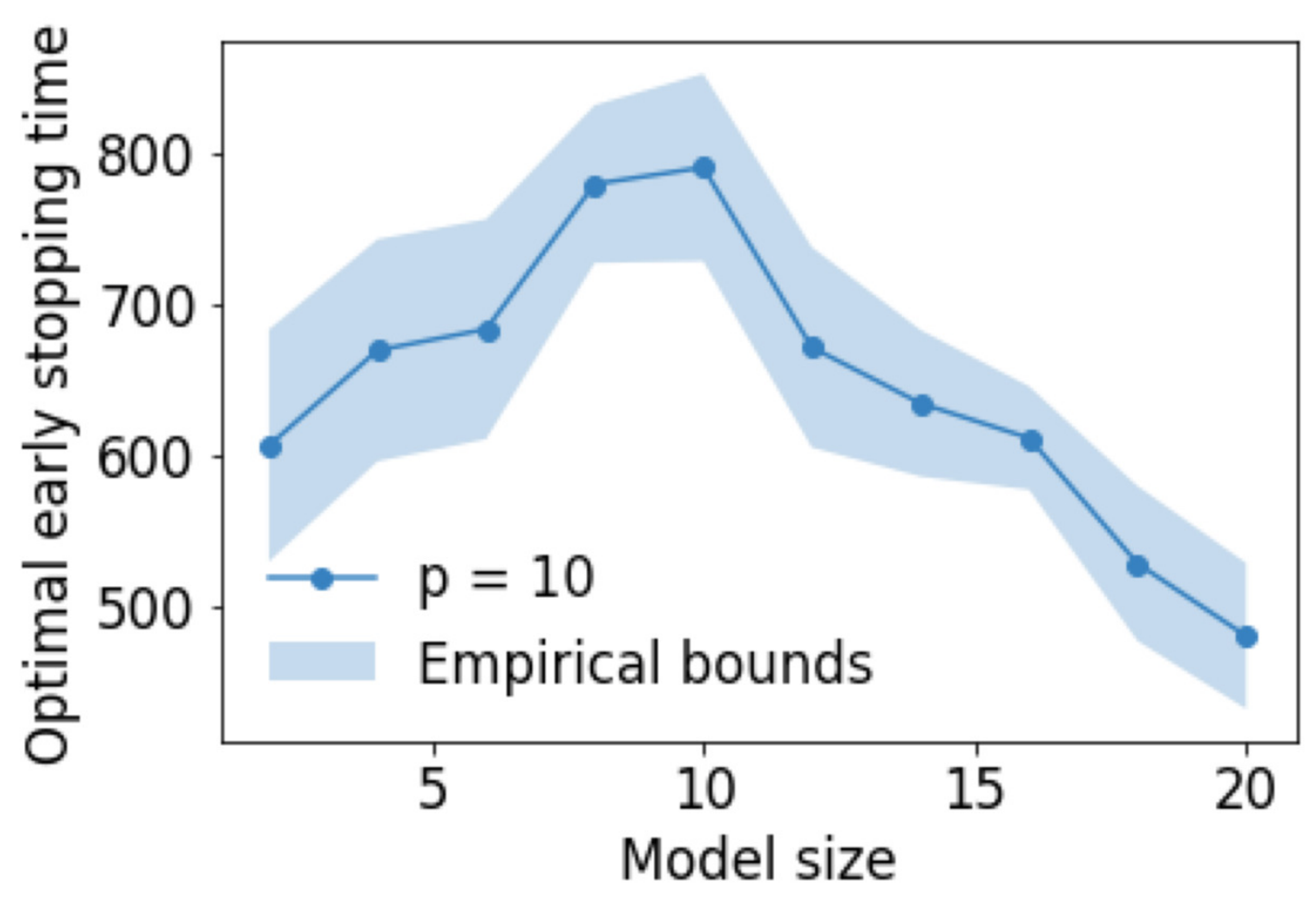}
  \caption{Optimal early stopping time vs. hidden layer width trained using two-layer ReLU networks on randomly generated points with $p=10$ informative features and 20\% label noise.}
  \label{fig:intro_fig2}
\end{minipage}

\end{figure*}

\paragraph{A conundrum on optimal early stopping} \cite{ali2019continuous} studied the regularization effect of early stopping on least square regression, but did not provide explicit characterization of the optimal early stopping time with respect to the model and data complexity. As a first step, we compute the optimal early stopping time of the model studied in \cite{ali2019continuous} under the setting where the inputs follow some Gaussian distribution. We discover that under the setting adopted therein, the optimal early stopping time increases along with the model size. However, this trend is contrary to what we observe in training neural networks. When training popular datasets such as MNIST, CIFAR-10/100 using deep neural networks, the optimal early stopping time decreases as the model size increases (see Figure \ref{fig:intro_fig}). These two contradicting trends indicate that previous theoretical models are not sufficient to explain what happens in practice.

\paragraph{Over-informative vs. under-informative parametrization} We notice that despite the huge models used in practice, the labels of the datasets are actually generated by low dimensional representations or features (See Section~\ref{sec:exp_deep}). Previous theoretical works have primarily focused on the cases where the model dimension is approximately the same or less than the number of features \citep{ belkin2020two, nakkiran2020optimal}. To complete the missing pieces in previous analysis, we propose a new setting where model size exceeds the number of features. We define such cases where more features than necessary can be extracted by the model as being \emph{over-informative}. This new setting demonstrates very different behavior from the previous theoretical model, which resolves the contradiction we previously observed.
We illustrate both theoretically and experimentally that this over-informative parametrization better reflects what happens in training models with excessive parameters. 

To see this, we first look at the last layer of a simplified neural network model. Let $\theta$ be the weight of the neural network's last layer and $\psi$ be some map given by all previous layers of the neural network. Consider the model given by 
\begin{equation}
\label{eq:two-layer}
    f_{\theta, \psi}(u) = \theta^\top \psi \Par{u},
\end{equation}
where $u$ is the input vector. Multi-layer structure enables neural networks to extract sophisticated features from the samples \citep{zeiler2014visualizing}. We thereby expect the matrix $\psi(u)$ in a trained model to be able to extract the important features of the sample $u$. 
The linear layer $v$ then weighs how much each feature contributes to the output and fits the corresponding label $y$. 
At initialization, the model extracts far more features than is necessary to fit the label $y$.
The reason is that sample $u$ contains noises that do not correspond to the label $y$ but are nevertheless extracted in $\psi(u)$. 
In the image classification example, features extracted from the unimportant parts of the images are not very useful in determining the labels.
Moreover, due to the over-parametrization nature of the neural networks, which typically have large hidden layer widths, a same feature can be extracted from multiple noisy angles.  
This setting, where more features than necessary are extracted, corresponds to the over-informative parametrization mentioned above. 

Our goal is to study how the optimal early stopping time changes when the model size and sample size changes. 
In general, when the model size increases, the model can extract more features. In the over-informative setting where the features extracted are already informative enough, extracting even more features can hardly bring any new information. Instead, due to random weight initialization and sample noise, extracting more feature can bring additional distraction to the training process. Contrary to the over-informative setting, the under-informative setting is the setting where not enough features are extracted, due to low sample quality or low model capacity. In such case, the features extracted are not informative enough to fit the label, so extracting more features can bring in new information, which helps the learning process. Due to the different roles extra features play in the over-informative and the under-informative setting, we expect early stopping, which is a regularization method, to behave differently in these two settings. Intuitively, in the over-informative setting, extra features can bring distraction so that the model can overfit more easily, so we need stronger regularization when we have more features. In the under-informative setting, the available information is limited, so the opposite is true. We need strong regularization when the number of available features is small to prevent overfitting. We show both theoretically and empirically that this hypothesis is true. Similarly, in both settings, when the sample size is small, we expect strong regularization to be beneficial. While many previous works focus on the under-informative setting, we study both settings and establish the differences between those two settings. We believe this over-informative and under-informative viewpoint and their differences can inspire future works in studying models with excessive parameters. 

We study the training process of a simplified model where the extracted features $\psi(u)$ are fixed and only the last layer is trained. This falls under the lazy-training regime. For a simple two-layer neural network, the number of features extracted is determined by the hidden layer width. For general neural networks, the relationship between the model size and the number of feature extracted can be more complicated. We believe model width is a stronger indicator on how many features are extracted than model depth since model depth relates more to the complexity of the features extracted. However, understanding the feature learning process of neural networks is still an open problem, so we are open to different answers to this question. In this simplified model, the problem is reduced to a linear model where the input is the features extracted from the samples. For the over-informative setting, the label of each sample can be generated by a low-dimensional transformation of the features extracted. For the under-informative setting, the features are not informative enough so that the features are given by a low-dimensional transformation of what actually generates the labels. We introduce the models we study formally in Section~\ref{sec:theory_over} and Section~\ref{sec:theory_under}. 

In the over-informative setting, we show the optimal early stopping time decreases as the model dimension increases or as the sample size decreases, similar to what we observe when training deep neural networks for the image classification tasks. On the other hand, in the under-informative setting, the optimal early stopping time increases as the model dimension increases or as the number of samples increases.
In addition to the theoretical results, we demonstrate experimentally that these trends persist when training various non-linear models. In particular, we are able to observe a change in the behavior of the early stopping time when the model dimension exceeds the feature number in non-linear models (e.g., see Figure \ref{fig:intro_fig2}). 

\paragraph{Generalization of the early-stopped models}Even for the same model, stopping at the optimal early stopping time versus after the training converges can lead to solutions with very different generalization properties. We also give a brief study on how the test risk of models changes as a function of other training parameters when stopped at the optimal early stopping time. In particular, it has been shown that without regularization, the test risk can exhibit the so-called ``double descent'' phenomenon \citep{belkin2018reconciling, nakkiran2019deep}, where the risk as a function of model size or sample size experience two distinct phases of descent. 
This phenomenon can make deciding the best model parameters challenging.
We demonstrate that early stopping helps mitigate double descent in multiple settings.

\subsection{Related work}

Regularization is widely used in training machine learning models to prevent overfitting. \citep{dobriban2018high, dobriban2020wonder, hastie2019surprises, kobak2020optimal, mei2019generalization, nakkiran2020optimal} studied the test risk of ridge regression. In addition, a lot of recent works studied the theoretical guarantee of implicit regularization in model training \citep{ derezinski2019exact,vaskevicius2019implicit, haochen2020shape, blanc2020implicit, razin2020implicit}. \cite{nakkiran2020optimal} studies the optimal $\ell_2$ regularization in the under-informative setting. In our work, we study regularization using early stopping, which gives a regularizer different from the $\ell_2$ regularization. More importantly, we focus on the over-informative setting, which better reflect what happens in large models, and establish the difference between the over-informative and under-informative setting. Early stopping time has been studied in some recent works \citep{li2020gradient,vavskevivcius2020statistical}. In particular, \cite{ali2019continuous,ali2020sgd} studied the risk of some early stopped linear models, but didn't characterize the exact early stopping time. Our work characterizes the exact optimal early time in the over-informative and under-informative settings and study how they change as a function of the sample and model complexity. 

The test risk of simplified machine learning models has been studied in a long line of works. A partial list of papers that studied linear models similar to ours includes \cite{bartlett2020benign, chen2020multiple, dobriban2018high, hastie2019surprises, montanari2019generalization, mitra2019understanding, muthukumar2020harmless}. Finally, our work studies the low rank structure of the feature representations. \cite{li2018measuring, dusenberry2020efficient} have also empirically explored and leveraged this property.

\subsection{Notation}
We use tilde notation $\tilde{O}$, $\tilde{\Theta}$, $\tilde{\Omega}$ to hide $\log$ factors in standard asymptotic notation. For any matrix $A$, $\exp(A)$ is defined as $\exp(A):=\sum_{n=0}^{\infty}\frac{A^n}{n!}$. 
\section{Main results}
\subsection{Over-informative parametrization}
\label{sec:theory_over}
In this section, we introduce and analyze formally the over-informative  setting, where the model size exceeds the number of features extracted. This setting completes the missing piece of previous theoretical analysis on large models with excessive parameters. For data $(x,y)$, the input $x\in \R^d$ is from a high dimensional space, which represents the over-informative features extracted. The label $y\in\R$ is generated from a low-dimensional representation $z\in \R^p$, $p\leq d$, which represents the true features of the sample. $z$ is mapped from $x$ by a low-dimensional transformation. We let $\so$ denote the over-informative parametrization setting studied in this section.

Formally, we consider the following linear regression problem with $d > p$. Let the covariate/input matrix $X\in \R^{n \times d}$ be a random matrix such that each entry of it is generated from $\mathcal{N}(0,1)$ independently. Let $x_1,...,x_n$ be the row vectors of the matrix $X$. For a semi-orthogonal matrix $P\in \R^{p\times d}$ with $p < d$ and some unknown parameter $\theta^*$, the response is generated by 
\begin{equation}
	\label{eq:y_generate}
	y_i = \inprod{Px_i}{\theta^*} + \eps_i,
\end{equation}
where the label noise $\eps_i \sim \mathcal{N}(0,\sigma^2)$ is independent of $x_i$ and is generated independently for each $(x_i,y_i)$.  Let $\eps \in \R^{n}$ be the vector with entries given by $\eps_i$.

For semi-orthogonal matrix $P$ and parameter $\theta^*$, let $\mathcal{D}_{P, \theta^*}$ be the joint distribution of a single data point $(x, y)$ given $P$ and $\theta^*$. For an estimator $\beta$, define the population risk  
\begin{equation*}
	R(\beta) = \E_{(\tx,\ty) \sim \mathcal{D}_{P, \theta^*}} \Brack{\Par{\inprod{\tx}{\beta} - \ty}^2}.
\end{equation*}
We consider the standard linear least squares problem
\begin{equation}
    \label{eq:lsp}
	\min_{\beta \in \R^d} \frac{1}{2n} \norm{y- X\beta}_2^2.
\end{equation} 

Applying gradient flow starting from $0$ on \eqref{eq:lsp} gives a continuous differential equation over time $t$, 
\begin{equation}
	\label{eq:gradflow}
	\frac{d}{dt}\beta_{X,y}(t) = \frac{X^\top}{n}(y - X\beta_{X,y}(t)).
\end{equation}
with an initial condition $\beta_{X,y}(0) = 0$. We can solve the differential equation exactly and obtain 
	\begin{equation}
		\label{eq:gradflowsol}
		\beta(t) = \Par{X^\top X}^+ \Par{I - \exp\Par{-tX^\top X/n}}X^\top y,
	\end{equation}
	for all $t\geq 0$ \citep{ali2019continuous}.

In practice, one would discretize the gradient flow and perform gradient descent. We discuss the error due to such discretization in Appendix~\ref{sec:discrete}.
\begin{remark}
 We assume the input data follows an isotropic Gaussian distribution. In practice, the feature extracted might not be independent, especially when the number of features is large so that some features can be redundant. However, our analysis only relies on the singular values of the data matrix $X$, so it is possible to relax the data assumptions on $X$ to bounds on the singular values of the data matrix. When the features are correlated, if the relationship between the singular values and the model/sample size stays the same, the trends we identify still hold. 
\end{remark}
\begin{remark}
The linear model has been widely used as the first step to study more complicated models as well due to its simple form. It is possible to extend our results to non-linear settings such as the random feature setting. One approach to extend to such lazy-training settings is to first locate the features extracted at initialization and then show how gradient flow progresses on the features extracted, similar to what we do. Extending to non-lazy training regime can be more challenging. In this paper, we stick to the simplest model that the observed phenomenon can be shown so that we can get a simple and explicit characterization of the optimal early stopping time. 
\end{remark}
For any dataset $y \in \R^n$ and $X\in \R^{n\times d}$, let $\beta_{X,y}(t)$ be the gradient flow solution of standard linear regression problem at time $t$. 
We further assume that $P\in \R^{p\times d}$ is a uniformly random semi-orthogonal matrix. 

 Our goal is to study the expected risk of the estimator $\beta_{X,y}(t)$ over $P$ and $\epsilon$. We denote the expected risk at time $t$ for dataset size $n$ as
\begin{equation*}
	\oR_{X,  \theta^*}(t) = \E_{P, \epsilon} \Brack{R(\beta_{X,y}(t))}.
\end{equation*}
We study the optimal early stopping time, the time $t$ that achieves the first local minimum of the expected risk. We characterize the optimal early stopping time for different parameter choices. 
We let $\topt$ denote the optimal stopping time, omitting $(X,  \theta^*)$ when clear from context,
\begin{align}
	&\topt(X,  \theta^*) =  \min t \notag\\ &\;\;s.t. \;\;\frac{d}{dt} 	\oR_{X,  \theta^*}(t')  < 0 \text{  for  } 0<  t' < t, \text{  and  } \frac{d}{dt} 	\oR_{X,  \theta^*}(t)  =0. 
\end{align}

With the gradient flow solution $\beta_{X,y}$, we can derive a high probability upper and lower bound on $\topt$.
\begin{restatable}[Optimal early stopping time in the over-informative parametrization setting]{theorem}{overoptstopping}
	\label{thm: over_optstopping}
	In setting $\so$, for a fixed parameter $\theta^*$ and noise variance $\sigma^2$, when $n \leq d$, let $\gamma = \Par{\frac{\sqrt n + \sqrt{2\log n}}{\sqrt d}}^2.$ For $\gamma \leq 1$, with probability at least $1-\frac 2 n$ over the randomness of $X$, the optimal early stopping time $\topt$ satisfies
	\begin{align*}
		\frac{n}{\Par{1 +\sqrt \gamma}^2d}\log\Par{1 + \frac{\Par{1-\sqrt \gamma}^2\norm{\theta^*}_2^2}{\sigma^2 }} \leq \topt \leq  \frac{n}{\Par{1-\sqrt \gamma}^2d}\log\Par{1 + \frac{\Par{1+\sqrt \gamma}^2\norm{\theta^*}_2^2}{\sigma^2 }} .
	\end{align*}
	When $n > d$, let $\gamma = \Par{\frac{\sqrt n}{\sqrt d + \sqrt {2 \log d}}}^2$. For $\gamma \geq 1$, with probability at least $1-\frac 2 d$ over the randomness of $X$, the optimal early stopping time $\topt$ satisfies
	\begin{align*}
	\frac 1 {\Par{1+\frac 1 {\sqrt \gamma}}^2}\log\Par{1 +\Par{1-\frac 1 {\sqrt \gamma}}^2 \frac{n\norm{\theta^*}_2^2}{d\sigma^2 }} \leq \topt 
	\leq\frac 1 {\Par{1-\frac 1 {\sqrt \gamma}}^2} \log\Par{1 +\Par{1+\frac 1 {\sqrt \gamma}}^2 \frac{n\norm{\theta^*}_2^2}{d\sigma^2 }}.
	\end{align*}
\end{restatable}
Note that our definitions of $\gamma$ differ slightly in the case $n \leq d$ and $n > d$, but they are both of order $\tilde{\Theta}\Par{\frac n d}$.
Theorem~\ref{thm: over_optstopping} gives an approximation of the optimal stopping time up to a constant as long as $\gamma$ is constantly bounded away from $1$. 
When $\gamma \leq \frac 1 4$ (roughly $n \leq 4d$), omitting the logarithmic terms, $\frac {9n}{4d} \leq \topt \leq \frac{2n}{d}$.
When $\gamma \geq 4$ (roughly $n \geq 4d$),  $\frac {n}{3(d+n)} \log \Par{1+\frac {n\norm{\theta^*}} {4d\sigma^2}}\leq \topt \leq \frac{5n}{d+n} \log\Par{1+3\frac {n\norm{\theta^*}} {d\sigma^2}}$.  Combining these two cases shows that when $\gamma$ is bounded away from $1$, $\topt$ is about ${\Theta}\Par{\frac n {d+n}\log \frac n d}$. When $\gamma $ is very close to $1$, the upper bound on $\topt$ in Theorem~\ref{thm: over_optstopping} can be loose. 
We show in Appendix~\ref{sec:asymptotic}, that in the asymptotic case where $n$ and $d$ are large and $\frac{\norm{\theta^*}}{\sigma^2}$ is lower bounded, the approximation $\topt = \tilde{\Theta}\Par{\frac n {d+n}}$ still holds when $\gamma$ is close to 1.
In the asymptotic analysis, we also achieve tighter constants in the bounds of the optimal early stopping time. 
In Appendix~\ref{sec:exp_appendix}, we give some numerical experiments on this model.

Theorem~\ref{thm: over_optstopping} shows that the optimal early stopping time is roughly $\tilde{\Theta}\Par{\frac n {n+d}}$, when $n\leq d$ and when $n$ and $d$ are around the same order. When $n\gg d$, the optimal early stopping time is roughly $\Theta(\log \frac n d)$.  It implies that {\bf when the number of features is smaller than the model dimension, optimal early stopping time increases as $n$ increases or $d$ decreases.} We show in Section~\ref{sec:exp_network_opt} and Section~\ref{sec:exp_twolayer} that training deep neural networks on large real datasets also follows these two trends. An intuitive explanation of such phenomenon is when the sample size is large, the model has access to more information, so the model can be trained longer before the overfitting starts. On the contrary, when $d$ increases, the number of parameters is large so that even small updates of each parameter can result in overfitting. In such case, the model can benefit from stronger early stopping regularization.

\subsection{Under-informative parametrization}
\label{sec:theory_under}
In this section, we study the under-informative parametrization setting, where the number of underlying features exceeds the model size, so the extracted features are not informative enough. This setting shows very different behavior from the over-informative setting. In this setting, the label $y$ is generated by some covariate $z$, which is in a high-dimensional space, but the model only has access to low-dimensional extracted features, which are given by a low dimensional projection of $z$. 

Formally, we consider the following linear regression problem with $d \leq p$. Let the covariate matrix $Z\in \R^{n \times p}$ be a random matrix such that each entry of it is generated from $\mathcal{N}(0,1)$ independently. Let $z_1,...,z_n$ be the row vectors of the matrix $Z$. For some unknown parameter $\theta^*$, the response is generated by 
\begin{equation}
	\label{eq:y_gene_under}
	y_i = \inprod{z_i}{\theta^*} + \eps_i.
\end{equation}
where the label noise $\eps_i \sim \mathcal{N}(0,\sigma^2)$ is independent of $z_i$ and $y_i$ and is independent for each sample. 
For a semi-orthogonal matrix $P\in \R^{p\times d}$ with $d \leq p$, let $X = ZP$. we apply gradient flow on $(X,y)$. For any semi-orthogonal matrix $P$ and parameter $\theta^*$, let $\mathcal{D}_{P, \theta^*}$ be the joint distribution of a single data point $(x, y)$. For an estimator $\beta$, define the population risk  
\begin{equation*}
	R(\beta) = \E_{(\tx,\ty) \sim \mathcal{D}_{P, \theta^*}} \Brack{\Par{\inprod{\tx}{\beta} - \ty}^2}.
\end{equation*}

Similar as in the over-informative setting, we consider the gradient flow solution \eqref{eq:gradflowsol} of the standard linear least square problem \eqref{eq:lsp}. For any dataset $y \in \R^n$ and $X\in \R^{n\times d}$ given by some $Z$, $P$ and $\eps$, let $\beta_{X,y}(t)$ be the solution of the gradient flow given in \eqref{eq:gradflowsol}. We further assume that $P\in \R^{p\times d}$ is a uniformly random semi-orthogonal matrix. 

We let $\su$ denote the above setting. $\su$ studies the case where we only have partial access to the features that determine the label $y$ and train the model on the limited available information.

Our goal is to study the expected risk of the estimator $\beta_{X,y}(t)$ over $Z$, $P$ and $\epsilon$.  We denote the expected risk over $Z$, $P$ and $\epsilon$ at time $t$ as
\begin{equation*}
	\oR_{\theta^*}(t) = \E_{Z, P, \epsilon} \Brack{R(\beta_{X,y}(t))}.
\end{equation*}
Note that here in $\su$, we take expectation on $Z$ when computing the expected risk. In $\so$, we do not take expectation on $X$ and instead get a high probability bound on $X$. This slight difference is due to the different methods we use in analyzing these two settings. 
We study the optimal early stopping time, the time $t$ that achieves the first local minimum of the expected risk. We characterize the optimal early stopping time for different choices of $p$, $n$ and $d$. 
We let $\topt$ denote the optimal stopping time, omitting $ \theta^*$ when clear from context,
\begin{align}
	&\topt(\theta^*) = \min t \notag\\
	&\;\;s.t. \;\; \frac{d}{dt} 	\oR_{  \theta^*}(t')  < 0 \text{  for  } 0<  t' < t, \text{  and  } \frac{d}{dt} 	\oR_{\theta^*}(t)  =0. 
\end{align}
We are able to give an approximation of the optimal stopping time $\topt$.
\begin{restatable}[Optimal early stopping time in under-informative setting]{theorem}{underoptstopping}
	\label{thm: under_optstopping}
	In setting $\su$, for a fixed parameter $\theta^*$ and noise variance $\sigma^2$, if $(8n+9d+16)\norm{\theta^*}_2^2\leq{p\sigma^2 + p\norm{\theta^*}_2^2}$, the optimal early stopping time $\topt$ satisfies
	\begin{equation*}
		\frac{{n\norm{\theta^*}_2^2}}{ 2\Par{p\sigma^2 + \Par{p-d}\norm{\theta^*}_2^2} }\leq t_\textup{opt}\leq  \frac{2n{\norm{\theta^*}_2^2}}{{p\sigma^2 + \Par{p-d}\norm{\theta^*}_2^2}}.
	\end{equation*}
\end{restatable}
Theorem~\ref{thm: under_optstopping} approximates the early optimal stopping $\topt$ up to a constant when $p$ is larger that $n$ and $d$ by a constant factor. For $d$ that is close to $p$, numerical experiments show that the approximation $\tilde{\Theta}\Par{\frac{{n\norm{\theta^*}_2^2}}{{p\sigma^2 + \Par{p-d}\norm{\theta^*}_2^2} }}$ can still hold. Theorem~\ref{thm: under_optstopping} implies that {\bf when the number of features is larger than the model dimension, optimal early stopping time increases as $n$ or $d$ increases.} An intuitive explanation is when the model is when the model size and the sample size are small, the model does not have enough information to train so that the model tends to fit the noise easier. In such case, regularization from early stopping can be beneficial.

\section{Generalization of the early-stopped models}
We give a brief discussion on the effect of early stopping on generalization. 
It has been observed empirically that without early stopping or other types of regularization, when the number of samples $n$ increases, the test error can experience double descent in presence of label noise. \cite{nakkiran2019deep} observed that optimal early stopping can possibly mitigate the double descent phenomenon empirically. In Proposition~\ref{thm:over_n_mono}, we bound the expected risk at the optimal stopping time for varying data size $n$ and explain this phenomenon theoretically. 
\begin{restatable}{proposition}{overnmono}
	\label{thm:over_n_mono}
   In setting $\so$, for a fixed parameter $\theta^*$ and noise variance $\sigma^2$. Let $\tear =  {\alpha n} /(n+d)$ for some constant $\alpha$. Let $\lambda_1 \geq ...\geq \lambda_d$ be the eigenvalues of the matrix $\frac 1 n X^\top X$. Then, 
   \begin{equation}
   	\label{eq: riskmono}
	\overline{R}_1  \leq \E_X [\oR_{X, \theta^*}(\tear)] \leq \oR_2 . 
    \end{equation}
   where $$\overline{R}_1 = \sigma^2 + \E_X\Brack{\sum_{i=1}^d\exp\Par{-2\alpha n \lambda_i/d}}\frac {\norm{\theta^*}^2}{d},$$ and $ \oR_2 =\oR_1 + \frac 1 2\alpha\sigma^2.$  $\overline{R}_1$ and $\oR_2$ decrease monotonically as $n$ increases.
\end{restatable}
Theorem~\ref{thm: over_optstopping} shows that the optimal early stopping time is roughly $\topt \approx \frac{\alpha n}{n+d}$ for a small constant $\alpha$ up the logarithmic factors. When $\frac{\norm{\theta}^*}{\sigma^2}$ is not too small, \eqref{eq: riskmono} gives a small region that bounds the expected risk at $\tear$. We are able to show that both the upper and the lower bounds of the region decrease as $n$ increases. This implies that { optimal early stopping can mitigate sample-wise double descent.} A related question is whether stopping at optimal early stopping time can mitigate the double descent due to increasing model size $d$. \cite{nakkiran2019deep} observed that double descent can still exist  at optimal early stopping.

For the under-informative parametrization setting, we can also approximately bound the risk at the optimal early stopping time. 

\begin{restatable}{proposition}{optriskunder}
	\label{thm:optrisk_under}
	In setting $\su$, for a fixed parameter $\theta^*$ and noise variance $\sigma^2$, let $\tear = \frac{{n\norm{\theta^*}_2^2}}{ {p\sigma^2 + \Par{p-d}\norm{\theta^*}_2^2} }$. Then, when $(8n+9d+16)\norm{\theta^*}_2^2\leq{p\sigma^2 + p\norm{\theta^*}_2^2}$, we have $$\overline{R}_1   \leq \oR_{\theta^*}( \tear) \leq   \overline{R}_2,$$
	where 
	$\overline{R}_1 = \sigma^2 + \norm{\theta^*}_2^2 - \frac { 2nd \norm{  \theta^*}_2^4}{p^2\sigma^2 + p\Par{p-d}\norm{\theta^*}_2^2} $ and $\overline{R}_2= \sigma^2 + \norm{\theta^*}_2^2 - \frac { 3nd \norm{  \theta^*}_2^4}{4\Par{p^2\sigma^2 + p\Par{p-d}\norm{\theta^*}_2^2}}. $
	
	Here, $ \overline{R}_1\geq 0.8\overline{R}_2$ and both $\overline{R}_1$ and $\overline{R}_2$ decrease when $d$ or $n$ increases.
	
\end{restatable}

Proposition~\ref{thm:optrisk_under} gives an upper bound and a lower bound on the risk at the approximated optimal early stopping time. Under the assumption $d$ and $n$ are small, $\oR_1$ and $\oR_2$ give a tight region ($ \overline{R}_1\geq 0.8\overline{R}_2$). Proposition~\ref{thm:optrisk_under} supports the observation that { risk at optimal early stopping time monotonically decreases as $n$ or $d$ increases}. 
We give additional numerical experiments for the results in this section in Section~\ref{sec:exp_appendix}. %

\section{Experiment}

\label{sec:experiment}

\begin{figure*}[h!]
\centering
\includegraphics[width=0.95\linewidth]{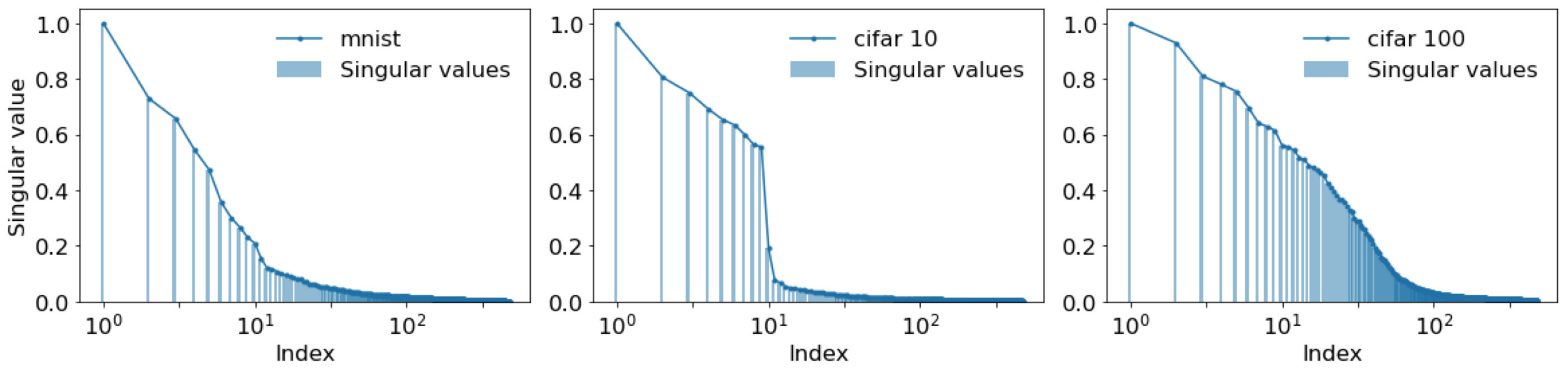}
\caption{Scaled singular values of representations obtained from MNIST, CIFAR-10 and CIFAR-100.}
\label{fig:eigval2}
\end{figure*}

\begin{figure*}[h!]
\centering

\includegraphics[width=0.98\linewidth]{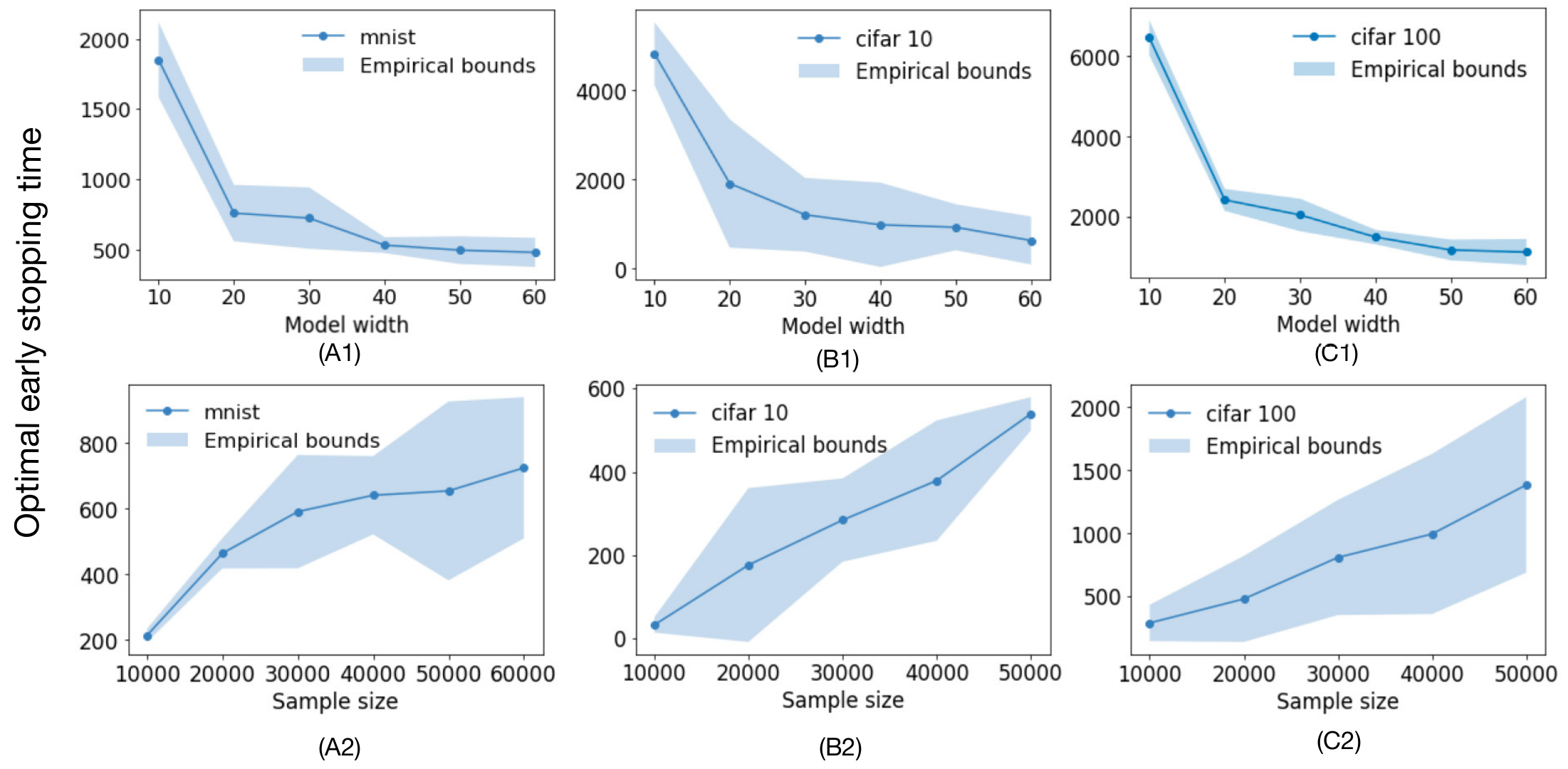}
\caption{In (A1) (B1) (C1), optimal early stopping time decreases with increasing model sizes for MNIST, CIFAR-10, and CIFAR-100. In (A2) (B2) (C2), optimal early stopping time increases with increasing sample sizes for MNIST, CIFAR-10, and CIFAR-100. For MNIST, we use a simple ReLU network with one hidden layer. We vary the hidden layer width over $d=[10, 20, ..., 60]$. For CIFAR-10/100, we use a family of ResNet networks with convolutional layers of widths $[d, 2d, 4d, 8d]$ for different layer depths. The sample size ranges from $n=[1e^5, 2e^5, ..., 6e^5]$ for MNIST, and $n=[1e^5, 2e^5, ..., 5e^5]$ for CIFAR-10/100. Inputs are normalized to $[-1, 1]$ with standard data-augmentation. We use stochastic gradient descent with cross-entropy loss, learning rate $\eta=0.1$ for MNIST, CIFAR-10, and $\eta=0.05$ for CIFAR-100, and minibatch size $B=512$. The optimal stopping time is computed as $t_{opt}=\eta T\frac{n}{B}$, where $T$ is the epoch with the
optimal early stopping performance.}
\label{fig:cifar_ci}

\end{figure*}

\begin{figure*}[h]
\centering

\includegraphics[width=0.65\linewidth]{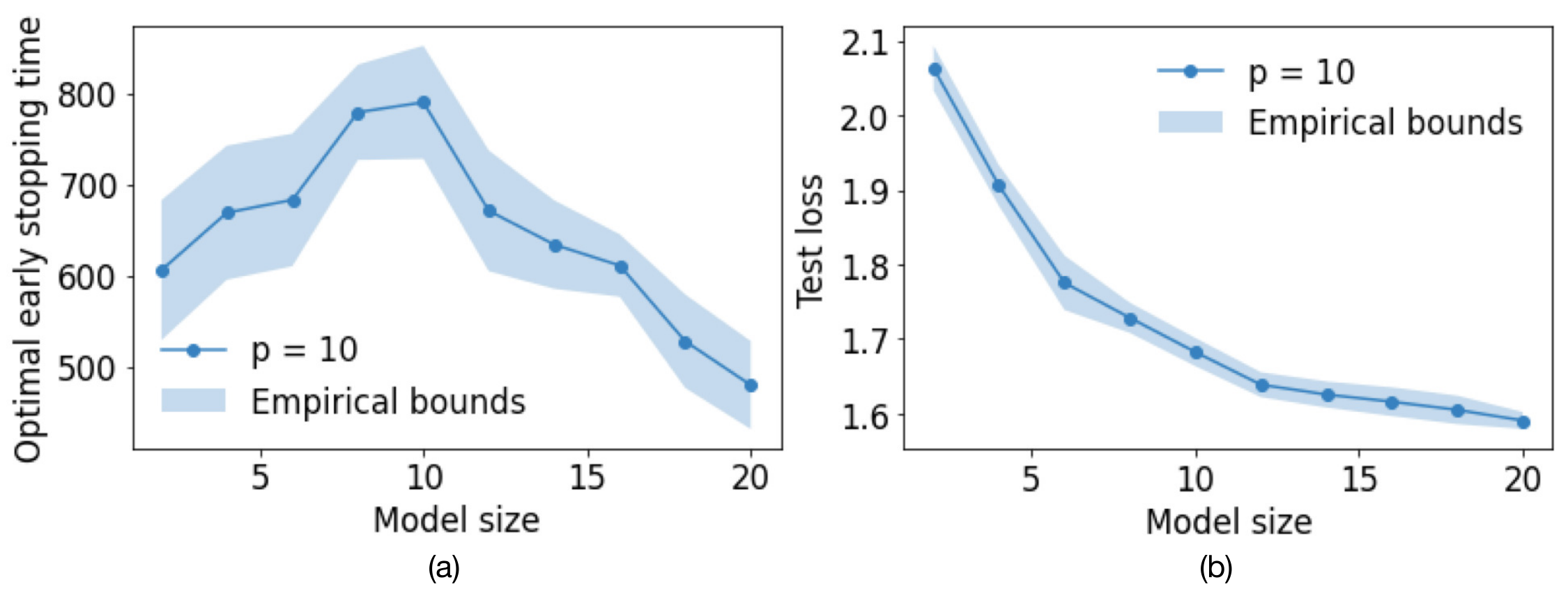}
\caption{(a) Optimal stopping time increases and then decreases as the model size increases. (b) Test loss decreases at the optimal early stopping time decreases as the model size increases.}
\label{fig:simulation_results}
\end{figure*}

 We first demonstrate that datasets MNIST, CIFAR-10 and CIFAR-100 have low-rank representations so that training large neural networks on them typically falls under the over-informative setting. We show that training on those dataset, the optimal early stopping time follows the pattern we identify for the overparametrization setting in Section~\ref{sec:exp_deep}. In Section~\ref{sec:exp_twolayer}, we show that for simple two-layer neural networks where we are able to observe under-informative behavior, we are able to observe a transition in the behavior of the optimal early stopping time when the model dimension exceeds the feature number. We provide additional experimental results omitted in this section in Appendix~\ref{sec:exp_appendix}.

\subsection{Classification using deep neural networks}
\label{sec:exp_deep}

\paragraph{Existence of low-rank structure in image classification. }

Many previous works observe that low-rank structures exist and can be leveraged in deep image classification tasks in various ways~\citep{van2008visualizing,dusenberry2020efficient}. 
This structure facilitates the self-supervised learning and transfer learning paradigm where image features are learned and fixed, reducing the complexity of image-to-label mapping. Then only the last layer is trained to adapt to and classify out-of-domain images~\citep{chen2020simple}.
We follow this approach and demonstrate the existence of the low-rank structures in Figure \ref{fig:eigval2}, where we plot the singular values of the representations of MNIST, CIFAR-10, and CIFAR-100. First, we train CNN/ResNet-60 \citep{he2016deep,he2016identity} to convergence with zero training loss (no label noise). Then, we obtain the representations produced by the trained neural networks on MNIST, CIFAR-10, and CIFAR-100. Even when trained with a very wide neural network, only the first 10/100 singular values of the representations are significant compared to the others. This suggests that low-rank feature representations of the datasets exist.

\paragraph{Over-informative parameterization behavior in image classification using neural networks}
\label{sec:exp_network_opt}

We show the optimal early stopping time of MNIST, CIFAR-10, and CIFAR-100. 
We train MNIST using a fully connected network, and CIFAR-10/100 using ResNet. We add 20$\%$ of label noise to all datasets. We provide additional results on different label noise levels in Appendix \ref{sec:exp_appendix}. To select the optimal stopping time, we first apply a moving averaging window with length 5 to mitigate the effect of mini-batch noise. After smoothing, we select the point with the optimal stopping time with the first minimal testing error (the lowest point of the first valley, ignoring small local minima).

We show in Figure \ref{fig:cifar_ci} the optimal early stopping time for all three datasets with varying $n$ and $d$. In (A1) (A2) (A3), we observe that the optimal early stopping time decreases with increasing model width. 
In (B1) (B2) (B3), we observe that the optimal early stopping time increases with larger sample sizes. 
Due to the existence of low-rank representations, Figure \ref{fig:cifar_ci} falls under the over-informative parametrization setting even for small $d$.
The observations are consistent with our theoretical analysis. When trained with ultra small $d<10$, the models show great instability. We provide additional experimental results supporting the under-informative setting in Section~\ref{sec:exp_twolayer}.

\subsection{Classification using two-layer neural networks}

\label{sec:exp_twolayer}

Training datasets like CIFAR-10/100 using ultra small models is unstable, so in this section, we use simpler datasets to study the under-informative setting. In this experiment, we use $\textit{make}\textunderscore\textit{classification}$ function in \textit{sklearn} \citep{scikit-learn} to generate clusters of points normally distributed around vertices of a high-dimensional hypercube and assigns an equal number of clusters to each class. We generate $2000$ samples (50-50 train-test split) of $20$ dimensions/features, but only $p=10$ dimensions/features of each sample are informative features. We add 20\% label noise to the dataset. The optimal stopping time is selected similarly as in Section \ref{sec:exp_deep}. 

We train the data using two-layer ReLU networks with varying hidden layer width $d$. The blue area in Figure \ref{fig:simulation_results} gives $0.5$ standard deviation of 6 runs and the blue line gives the mean. In Figure \ref{fig:simulation_results}a, when $d\leq p=10$, we observe the under-informative behavior where the optimal early stopping time increases as the model size $d$ increases. When $d\geq 10 = p$, the optimal stopping time decreases with $d$. We can further observe that the test loss decreases with increasing model sizes at the optimal early stopping time. 

\section{Conclusion and future work}

In this paper, we propose a new model of over-informative parametrization
to study the optimal early stopping time and establish its difference from the under-informative parametrization. This model captures the phenomenon that the model size usually exceeds the number of important underlying features in practice. We believe this model can be of independent interest. One future direction is to study the performance of other optimizers or techniques on this model. 
Moreover, we give an explicit characterization of the early stopping time and the resulting model. Future works can explore other regularization properties of this early stopped model, e.g., whether this early stopping helps when there is a domain shift in the testing dataset. 

Another direction worth exploring is to extend our results to nonlinear settings. Extending our result to the lazy training regime of neural networks can be relatively straightforward. However, theoretical analysis of the non-lazy training is still a largely open question, so extending this result to non-lazy setting can be more challenging. To extend this result to the non-lazy setting for a fully nonlinear neural network model, one needs to study how the features extracted change and how the weights on the features change at the same time.

\section*{Acknowledgments}
This work is supported in part by the National Science Foundation Grants NSF-SCALE MoDL(2134209) and NSF-CCF-2112665 (TILOS), the U.S. Department Of Energy, Office of Science, and the Facebook Research award. 
\bibliography{refs.bib}
\bibliographystyle{plainnat}

\newpage

\appendix

\section{Deferred Proof}
\subsection{Proof of Theorem~\ref{thm: over_optstopping}}

We first compute the expected risk in the setting $\so$. 

\begin{restatable}{lemma}{overrisk}
	\label{lem:overrisk}
	For all $n, p, d \in \mathbb{N}$ such that $p \leq d$, let $X\in \R^{n\times d}$ and $\theta^* \in \R^p$ be fixed. Let $\lambda_1 \geq ...\geq \lambda_d$ be the eigenvalues of the matrix $\frac 1 n X^\top X$. Then, 
	\begin{align}
		\label{eq:overrisk}
		\oR_{X, \theta^*}(t) = \sigma^2 + \sum_{i=1}^d \exp\Par{-2t\lambda_i}\frac { \norm{\theta^*}_2^2}d +  \frac{\sigma^2}{n} \sum_{i=1}^d \indi\Brace{\lambda_i \neq 0} \frac 1{\lambda_i} \Par{1-\exp(-t\lambda_i)}^2.
	\end{align}
\end{restatable}
\begin{proof}
	We first compute $R(\beta)$ for any $\beta \in \R^d$ and a fixed $\theta^*$. Let $\beta^* = P^\top \theta^*$.
	\begin{equation}
		\label{eq:computerisk}
		\begin{aligned}
			R(\beta) & = \E_{(\tx,\ty) \sim \mathcal{D}_{P, \theta^*}} \Brack{\Par{\inprod{\tx}{\beta} - \ty}^2} \\
			& = \E_{(\tx,\ty) \sim \mathcal{D}_{P, \theta^*}} \Brack{\Par{\inprod{\tx}{\beta} - {\inprod{P\tx}{\theta^*} - \teps}}^2}\\
			& = \E_{(\tx,\ty) \sim \mathcal{D}_{P, \theta^*}} \Brack{\Par{\inprod{\tx}{\beta - \beta^*}  - \teps}^2}\\
			& = \sigma^2 + \norm{\beta - \beta^*}_2^2.
		\end{aligned} 
	\end{equation}
	The first step follows from the definition of $R $. The second step follows from \eqref{eq:y_generate}. The third step follows from $\inprod{P\tx}{\theta^*} = \inprod{\tx}{P^\top\theta^*}$. The last step follows from $\teps \sim \mathcal{N}(0,\sigma^2)$, $\tx \sim \mathcal{N}(0, \id_d)$ and $\teps$ is independent of $\tx$.
	
	Next, we compute the expected risk $\oR_{X, \theta^*}(t)$.
	\begin{equation}
		\label{eq:computerisk2}
		\begin{aligned}
			\oR_{X,  \theta^*}(t) & =  \sigma^2 +  \E_{P, \epsilon}[\norm{\beta_{X,y}(t)- \beta^*}_2^2]\\
			& =  \sigma^2 +  \E_{P, \epsilon}\Brack{ \norm{\Par{X^\top X}^+ \Par{I - \exp\Par{-tX^\top X/n}}X^\top \Par{X\beta^* + \epsilon} - \beta^*}^2} \\
			& =  \sigma^2 +  \E_{P}\Brack{ \norm{\Par{X^\top X}^+ \Par{I - \exp\Par{-tX^\top X/n}}X^\top {X\beta^* } -\beta^*}_2^2} \\
			& +\sigma^2  \norm{\Par{X^\top X}^+ \Par{I - \exp\Par{-tX^\top X/n}}X^\top  }_F^2
		\end{aligned}
	\end{equation} 
	Here, we use $\norm{\cdot}_F$ to denote the Frobenius norm.
	The first step follows from \eqref{eq:computerisk}. The second step follows from \eqref{eq:gradflowsol}. The third step follows from $\eps\sim \mathcal{N}(0, \id_n)$ is independent of $X$ and $\beta^*$.
	Now let $\frac 1 n X^\top X = U\Lambda U^\top$ be the eigenvalue decomposition of $\frac 1 n X^\top X$. Let $\lambda_1 \geq ...\geq \lambda_d$ be the eigenvalues of $\frac 1 n X^\top X$. Then, 
	\begin{equation}
		\label{eq:computerisk3}
		\begin{aligned}
			& \E_{P}\Brack{ \norm{\Par{X^\top X}^+ \Par{I - \exp\Par{-tX^\top X/n}}X^\top {X\beta^* } -\beta^*}_2^2} \\
			& = \E_{P}\Brack{ \norm{\Par{X^\top X}^+\Par{X^\top X} \Par{I - \exp\Par{-tX^\top X/n}} {\beta^* } -\beta^*}_2^2}\\
			& = \E_{P}\Brack{ \norm{ \Par{I - \exp\Par{-tX^\top X/n}} {\beta^* } -\beta^*}_2^2}\\
			& = \E_{P}\Brack{ \norm{{U \exp\Par{-t\Lambda}} {U^\top P^\top\theta^* } }_2^2}= \sum_{i=1}^d \exp\Par{-2t\lambda_i} \frac {\norm{\theta^*}_2^2} d. 
		\end{aligned}
	\end{equation}
	The first step follows from $X^\top X$, $\Par{X^\top X}^+$ and $I-\exp\Par{-tX^\top X/n }$ are simultaneously diagonalizable. The second step follows from $I-\exp\Par{-tX^\top X/n }$ is already in row space of $X^\top X$. The third step follows from writing $\frac 1 n X^\top X$ as $U\Lambda U^\top $. The last step follows from $P$ is a random semi-orthogonal matrix.  Similarly, we have
	\begin{equation}
		\label{eq:computerisk4}
		\begin{aligned}
			& { \norm{\Par{X^\top X}^+ \Par{I - \exp\Par{-tX^\top X/n}}X^\top  }_F^2} \\
			& = \tr\Par{\Par{X^\top X}^+ \Par{I - \exp\Par{-tX^\top X/n}}X^\top X\Par{I - \exp\Par{-tX^\top X/n}}\Par{X^\top X}^+}\\
			& =  \tr\Par{ \Par{I - \exp\Par{-tX^\top X/n}}^2\Par{X^\top X}^+}\\
			& = \frac{1}{n} \sum_{i=1}^d \indi\Brace{\lambda_i \neq 0} \frac 1{\lambda_i} \Par{1-\exp(-t\lambda_i)}^2.
		\end{aligned}
	\end{equation}
	Combining \eqref{eq:computerisk2}, \eqref{eq:computerisk3}
	and \eqref{eq:computerisk4} completes the proof.
\end{proof}

We state a useful concentration result.
\begin{lemma}
	\label{lem:eigbounds}
	Let $X\in \R^{n \times d}$, $d \geq n$ be a matrix whose entries are independent Guassian random variables following $\mathcal{N}(0,1)$. Let $\lambda_1\geq ... \geq \lambda_d$ be the $d$ eigenvalues of the matrix $\frac 1 n X^\top X$. Let $\gamma = \Par{\frac{\sqrt n + \sqrt{2\log n}}{\sqrt d}}^2$. For $\gamma \leq 1 $,  $\lambda_{n+1} = ...= \lambda _d = 0$, and with probability at least $1-{\frac 2 n }$,
	\begin{equation*}
		\frac d n \Par{1 -\sqrt \gamma}^2 \leq \lambda_n \leq \lambda_1 \leq \frac d n \Par{1 +\sqrt \gamma}^2.
	\end{equation*}
\end{lemma}
\begin{proof}
	$\lambda_{n+1} = ...= \lambda _d = 0$ follows from the matrix  $\frac 1 n X^\top X$ is at most rank $n$.
	By Lemma~\ref{lem:eig_concentration}, with probability at least $1-\frac 2 n$,
	\begin{equation*}
		\frac 1 n\Par{\sqrt{d} - \sqrt{n} - \sqrt{2\log n }}^2\leq  \lambda_n \leq \lambda_1
		\leq \frac 1 n\Par{\sqrt{d} + \sqrt{n} +  \sqrt{2\log n}}^2.
	\end{equation*}
	For $\gamma = \Par{\frac{\sqrt n + \sqrt{2\log n}}{\sqrt d}}^2$, $	\frac 1 n\Par{\sqrt{d} - \sqrt{n} - \sqrt{2\log n }}^2 = \frac d n \Par{1 -\sqrt \gamma}^2$ and $\frac 1 n\Par{\sqrt{d} + \sqrt{n} +  \sqrt{2\log n}}^2 = \frac d n \Par{1 +\sqrt \gamma}^2$, which completes the proof. 
\end{proof}

\begin{lemma}
	\label{lem:eigbounds2}
	Let $X\in \R^{n \times d}$, $d \leq n$  be a matrix whose entries are independent Guassian random variables following $\mathcal{N}(0,1)$. Let $\lambda_1\geq ... \geq \lambda_d$ be the $d$ eigenvalues of the matrix $\frac 1 n X^\top X$. Let  $\gamma = \Par{\frac{\sqrt n}{\sqrt d + \sqrt {2 \log d}}}^2$. For $\gamma \geq 1$,  and with probability at least $1-{\frac 2 d}$,
	\begin{equation*}
		\Par{1-\frac 1 {\sqrt \gamma}}^2\leq \lambda_d \leq \lambda_1 \leq \Par{1+\frac 1 {\sqrt \gamma}}^2.
	\end{equation*}
\end{lemma}
\begin{proof}
	By Lemma~\ref{lem:eig_concentration}, with probability at least $1-\frac 2 d$,
	\begin{equation*}
		\frac 1 n \Par{\sqrt n - \sqrt d - \sqrt{2\log d}}^2 \leq \lambda_d \leq \lambda_1 \leq \frac 1 n \Par{\sqrt n + \sqrt d + \sqrt{2\log d}}^2.
	\end{equation*}
	For $\gamma = \Par{\frac{\sqrt n}{\sqrt d + \sqrt {2 \log d}}}^2$, $ \frac 1 n \Par{\sqrt n - \sqrt d - \sqrt{2\log d}}^2 =  \Par{1-\frac 1 {\sqrt \gamma}}^2$ and $\frac 1 n \Par{\sqrt n + \sqrt d + \sqrt{2\log d}}^2= \Par{1+\frac 1 {\sqrt \gamma}}^2$.
\end{proof}
With Lemma~\ref{lem:eigbounds} and Lemma~\ref{lem:eigbounds2}, we are able to bound the optimal stopping time.
\overoptstopping*
\begin{proof}[Proof of Theorem~\ref{thm: over_optstopping}]
	We first compute the derivative of the expected risk. By Lemma~\ref{lem:overrisk}, 
	\begin{equation*}
		\begin{aligned}
			\frac d {dt} \oR_{X, \theta^*}(t) & =	\frac d {dt} \Brack{ \sigma^2 + \sum_{i=1}^d \exp\Par{-2t\lambda_i}\frac { \norm{\theta^*}_2^2}d +  \frac{\sigma^2}{n} \sum_{i=1}^d \indi\Brace{\lambda_i \neq 0}  \frac 1{\lambda_i} \Par{1-\exp(-t\lambda_i)}^2} \\
			& =  -\sum_{i=1}^d 2\lambda_i \exp\Par{-2t\lambda_i}\frac { \norm{\theta^*}_2^2}d +\frac{2\sigma^2}{n} \sum_{i=1}^d  \Par{\exp(-t\lambda_i)-\exp(-2t\lambda_i)}.
		\end{aligned}
	\end{equation*}
	
	Let $$\oR^{(i)}_{X,  \theta^*}(t) =  \exp\Par{-2t\lambda_i}\frac { \norm{\theta^*}_2^2}d +\frac{\sigma^2}{n}  \indi\Brace{\lambda_i \neq 0}  \frac 1{\lambda_i} \Par{1-\exp(-t\lambda_i)}^2.$$
	We first consider the case $d > n$. For $i$ such that $\lambda_i = 0$,  $ \frac d {dt}\oR^{(i)}_{X,  \theta^*}(t)=0$ for all $t$. For $i$ such that $\lambda_i \neq 0$, we can show that if $t <  \frac{n}{\Par{1 +\sqrt \gamma}^2d}\log\Par{1 + \frac{\Par{1-\sqrt \gamma}^2\norm{\theta}_2^2}{\sigma^2 }}$, $\frac d {dt}\oR^{(i)}_{X,  \theta^*}(t) < 0$. Let $\alpha = \exp(-t\lambda_i)$, then $$\frac d {dt}\oR^{(i)}_{X,  \theta^*}(t) = -\Par{\frac {2 \norm{\theta^*}_2^2  }d\lambda_i +\frac{2\sigma^2}{n}}\alpha^2 +\frac{2\sigma^2}{n}\alpha.$$
	When $\alpha > \frac{\sigma^2 d}{\sigma^2 d + \norm{\theta^*}^2 n \lambda_i}$, $\frac d {dt}\oR^{(i)}_{X,  \theta^*}(t) < 0$. When $t < \log \Par{ \frac{\sigma^2 d + \norm{\theta^*}^2 n \lambda_i}{\sigma^2 d}}/ {\lambda_i}$, $\alpha > \frac{\sigma^2 d}{\sigma^2 d + \norm{\theta^*}^2 n \lambda_i}$.     By Lemma~\ref{lem:eigbounds}, with probability at least $1-{\frac 2 n}$,
	\begin{equation*}
		\frac d n \Par{1 -\sqrt \gamma}^2 \leq \lambda_n \leq \frac d n \Par{1 +\sqrt \gamma}^2,
	\end{equation*} where $\gamma = \Par{\frac{\sqrt n + \sqrt{2\log n}}{\sqrt d}}^2$. Then, when $t <  \frac{n}{\Par{1 +\sqrt \gamma}^2d}\log\Par{1 + \frac{\Par{1-\sqrt \gamma}^2\norm{\theta}_2^2}{\sigma^2 }} $, $\frac d {dt}\oR^{(i)}_{X,  \theta^*}(t) < 0$ for all $i$ such that $\lambda_i \neq 0$, which shows $\topt \geq  \frac{n}{\Par{1 +\sqrt \gamma}^2d}\log\Par{1 + \frac{\Par{1-\sqrt \gamma}^2\norm{\theta}_2^2}{\sigma^2 }} $.
	Similarly, we can upper bound $\topt$. For $i$ such that $\lambda_i \neq 0$, when $0 < \alpha <  \frac{\sigma^2 d}{\sigma^2 d + \norm{\theta^*}^2 n \lambda_i}$, $\frac d {dt}\oR^{(i)}_{X,  \theta^*}(t) > 0$. When $t > \log \Par{ \frac{\sigma^2 d + \norm{\theta^*}^2 n \lambda_i}{\sigma^2 d}}/ {\lambda_i}$, $\alpha < \frac{\sigma^2 d}{\sigma^2 d + \norm{\theta^*}^2 n \lambda_i}$. Then, when $t > \frac{n}{\Par{1-\sqrt \gamma}^2d}\log\Par{1 + \frac{\Par{1+\sqrt \gamma}^2\norm{\theta}_2^2}{\sigma^2 }} $, $\frac d {dt}\oR^{(i)}_{X,  \theta^*}(t) > 0$ for all $i$ such that $\lambda_i \neq 1$, which shows $\topt \leq  \frac{n}{\Par{1-\sqrt \gamma}^2d}\log\Par{1 + \frac{\Par{1+\sqrt \gamma}^2\norm{\theta}_2^2}{\sigma^2 }} $.
	
	Next, we consider the case when $n> d$.
	For $i$ such that $\lambda_i = 0$,  $ \frac d {dt}\oR^{(i)}_{X,  \theta^*}(t)=0$ for all $t$. For $i$ such that $\lambda_i \neq 0$, we can show that if $t < \frac 1 {\Par{1+\frac 1 {\sqrt \gamma}}^2}\log\Par{1 +\Par{1-\frac 1 {\sqrt \gamma}}^2 \frac{n\norm{\theta}_2^2}{d\sigma^2 }}$, $\frac d {dt}\oR^{(i)}_{X,  \theta^*}(t) < 0$. Let $\alpha = \exp(-t\lambda_i)$, then $$\frac d {dt}\oR^{(i)}_{X,  \theta^*}(t) = -\Par{\frac {2 \norm{\theta^*}_2^2  }d\lambda_i +\frac{2\sigma^2}{n}}\alpha^2 +\frac{2\sigma^2}{n}\alpha.$$
	When $\alpha > \frac{\sigma^2 d}{\sigma^2 d + \norm{\theta^*}^2 n \lambda_i}$, $\frac d {dt}\oR^{(i)}_{X,  \theta^*}(t) < 0$. When $t < \log \Par{ \frac{\sigma^2 d + \norm{\theta^*}^2 n \lambda_i}{\sigma^2 d}}/ {\lambda_i}$, $\alpha > \frac{\sigma^2 d}{\sigma^2 d + \norm{\theta^*}^2 n \lambda_i}$.     By Lemma~\ref{lem:eigbounds2}, with probability at least $1-\frac 2 d$,
	\begin{equation*}
		\Par{1-\frac 1 {\sqrt \gamma}}^2\leq \lambda_d \leq \lambda_1 \leq \Par{1+\frac 1 {\sqrt \gamma}}^2,
	\end{equation*}
	where $\gamma = \Par{\frac{\sqrt n}{\sqrt d + \sqrt {2 \log d}}}^2$. Then, when $t < \frac 1 {\Par{1+\frac 1 {\sqrt \gamma}}^2}\log\Par{1 +\Par{1-\frac 1 {\sqrt \gamma}}^2 \frac{n\norm{\theta}_2^2}{d\sigma^2 }} $, $\frac d {dt}\oR^{(i)}_{X,  \theta^*}(t) < 0$ for all $i$ such that $\lambda_i \neq 0$, which shows $\topt \geq \frac 1 {\Par{1+\frac 1 {\sqrt \gamma}}^2}\log\Par{1 +\Par{1-\frac 1 {\sqrt \gamma}}^2 \frac{n\norm{\theta}_2^2}{d\sigma^2 }} $.
	Similarly, we can upper bound $\topt$. For $i$ such that $\lambda_i \neq 0$, when $0 < \alpha <  \frac{\sigma^2 d}{\sigma^2 d + \norm{\theta^*}^2 n \lambda_i}$, $\frac d {dt}\oR^{(i)}_{X,  \theta^*}(t) > 0$. When $t > \log \Par{ \frac{\sigma^2 d + \norm{\theta^*}^2 n \lambda_i}{\sigma^2 d}}/ {\lambda_i}$, $\alpha < \frac{\sigma^2 d}{\sigma^2 d + \norm{\theta^*}^2 n \lambda_i}$. Then, when $t > \frac 1 {\Par{1-\frac 1 {\sqrt \gamma}}^2} \log\Par{1 +\Par{1+\frac 1 {\sqrt \gamma}}^2 \frac{n\norm{\theta}_2^2}{d\sigma^2 }} $, $\frac d {dt}\oR^{(i)}_{X,  \theta^*}(t) > 0$ for all $i$ such that $\lambda_i \neq 0$, which shows $\topt \leq \frac 1 {\Par{1-\frac 1 {\sqrt \gamma}}^2} \log\Par{1 +\Par{1+\frac 1 {\sqrt \gamma}}^2 \frac{n\norm{\theta}_2^2}{d\sigma^2 }}$.
\end{proof}

\subsection{Proof of Theorem~\ref{thm: under_optstopping}}
We first compute the expected risk $\oR_{\theta^*}(t)$.
\begin{lemma}
	\label{lem:underrisk}
	For all $n, p, d \in \mathbb{N}$ such that $p \geq d$, let $\theta^* \in \R^p$ be fixed. Let $\mathcal{D}_\Lambda$ be the distribution of eigenvalues of the matrix $\frac 1 n \tilde{X}^\top \tilde{X}$, where each entry of the matrix $\tilde{X}\in \R^{n \times d}$ are i.i.d standard normal random variable.Then, 
	\begin{equation*}
		\begin{aligned}
			\oR_{\theta^*}(t) = \sigma^2 + \Par{1-\frac d p} \norm{\theta^*}_2^2 + \E_{\Lambda\sim\mathcal{D}_\Lambda}\Brack{\sum_{i=1}^d \exp\Par{-2t\lambda_i} \frac {\norm{  \theta^*}_2^2} p} + \\
			\E_{\Lambda\sim\mathcal{D}_\Lambda}\Brack{ \frac{1}{n} \sum_{i=1}^d\indi\Brace{\lambda_i \neq 0}  \frac 1{\lambda_i} \Par{1-\exp(-t\lambda_i)}^2 \Par{\sigma^2 + \frac{p-d}{p}\norm{\theta^*}_2^2}}.
		\end{aligned}
	\end{equation*}
\end{lemma}
\begin{proof}
	We first compute $R(\beta)$ for any $\beta \in \R^d$ and a fixed $\theta$. Let $\beta^* = P^\top \theta^*$, which minimizes the population risk $R(\beta)$. 
	\begin{equation}
		\label{eq:computerisk_under}
		\begin{aligned}
			R(\beta) & = \E_{(\tx,\ty) \sim \mathcal{D}_{P, \theta^*}} \Brack{\Par{\inprod{\tx}{\beta} - \ty}^2} \\
			& = \E_{(\tx,\ty) \sim \mathcal{D}_{P, \theta^*}} \Brack{\Par{\inprod{\tz}{P\beta} - {\inprod{\tz}{\theta^*} - \teps}}^2}\\
			& = \sigma^2 + \norm{P\beta - \theta^*}_2^2\\
			& = \sigma^2 + \norm{\theta^* -P \beta^*}_2^2 +  \norm{P\beta -P \beta^*}_2^2  + 2\inprod{P(\beta - \beta^*)}{\theta^* - P\beta^*}\\
			& =  \sigma^2 + \norm{\theta^* -P \beta^*}_2^2 +  \norm{\beta -\beta^*}_2^2 
		\end{aligned} 
	\end{equation}
	The first step follows from the definition of $R(\beta) $. The second step follows from \eqref{eq:y_gene_under}. The third step follows from $\eps \sim \mathcal{N}(0,\sigma^2)$, $\tz \sim \mathcal{N}(0, \id_d)$ and $\teps$ is independent of $\tx$. The last two steps follows from expanding and $P^\top \theta^* = \beta^* = P^\top P\beta^*$.
	
	Next, we compute the expected risk $\oR_{ \theta^*}(t)$. Since $P$ is a random semi-orthogonal matrix,  $$\E_P\Brack{\norm{\theta^* -P \beta^*}_2^2} = \E_P\Brack{\norm{\theta^* -P P^\top \theta^*}_2^2} = \Par{1-\frac d p} \norm{\theta^*}_2^2.$$ For the term $\norm{\beta -\beta^*}_2^2 $, 
	\begin{equation}
		\label{eq:computerisk2_under}
		\begin{aligned}
			&\E_{Z, P, \epsilon}\Brack{\norm{\beta -\beta^*}_2^2}\\
			& =    \E_{Z, P, \epsilon}\Brack{ \norm{ \Par{X^\top X}^+ \Par{I - \exp\Par{-tX^\top X/n}}X^\top \Par{X\beta^* + Z(I-PP^\top)\theta^* + \epsilon} - \beta^*}^2} \\
			& =  \E_{Z, P}\Brack{  \norm{ \Par{X^\top X}^+ \Par{I - \exp\Par{-tX^\top X/n}}X^\top {X\beta^* } -\beta^*}_2^2} \\
			& +\E_{Z, P, \epsilon}\Brack{ \norm{ \Par{X^\top X}^+ \Par{I - \exp\Par{-tX^\top X/n}}X^\top{ Z(I-PP^\top)\theta^*} }^2} \\
			& +\sigma^2 { \norm{\Par{X^\top X}^+ \Par{I - \exp\Par{-tX^\top X/n}}X^\top  }_F^2} 
		\end{aligned}
	\end{equation} 
	The first step follows from \eqref{eq:gradflowsol}. The second step follows from $\eps\sim \mathcal{N}(0, \id_n)$ is independent of $X$ and $\beta^*$ and $I - PP^\top$ projects to orthogonal complement of the column space of $P$.
	
	Now let $\frac 1 n X^\top X = U\Lambda U^\top$ be the eigenvalue decomposition of $\frac 1 nX^\top X$. We have
	\begin{equation}
		\label{eq:computerisk3_under}
		\begin{aligned}
			& \E_{Z,P}\Brack{ \norm{\Par{X^\top X}^+ \Par{I - \exp\Par{-tX^\top X/n}}X^\top {X\beta^* } -\beta^*}_2^2} \\
			& = \E_{Z,P}\Brack{ \norm{\Par{X^\top X}^+\Par{X^\top X} \Par{I - \exp\Par{-tX^\top X/n}} {\beta^* } -\beta^*}_2^2}\\
			& = \E_{Z,P}\Brack{ \norm{ \Par{I - \exp\Par{-tX^\top X/n}} {\beta^* } -\beta^*}_2^2}\\
			& = \E_{Z,P}\Brack{ \norm{{U \exp\Par{-t\Lambda}} {U^\top P^\top\theta^* } }_2^2}\\
			& = \E_{\Lambda,P}\Brack{\sum_{i=1}^d\indi\Brace{\lambda_i \neq 0}  \exp\Par{-2t\lambda_i} \frac {\norm{ P^\top \theta^*}_2^2} d}\\
			& = \E_{\Lambda}\Brack{\sum_{i=1}^d\indi\Brace{\lambda_i \neq 0}  \exp\Par{-2t\lambda_i} \frac {\norm{  \theta^*}_2^2} p}
		\end{aligned}
	\end{equation}
	The first step follows from $X^\top X$, $\Par{X^\top X}^+$ and $I-\exp\Par{-tX^\top X/n }$ are simultaneously diagonalizable. The second step follows from $I-\exp\Par{-tX^\top X/n }$ is already in row space of $X^\top X$. The third step follows from writing $\frac 1 n X^\top X$ as $U\Lambda U^\top $. The last two step follows from $X$ and $ P$ is independent by rotational invariance of $X$. Similarly, we have
	\begin{equation}
		\label{eq:computerisk4_under}
		\begin{aligned}
			& { \norm{\Par{X^\top X}^+ \Par{I - \exp\Par{-tX^\top X/n}}X^\top  }_F^2} \\
			& = \tr\Par{\Par{X^\top X}^+ \Par{I - \exp\Par{-tX^\top X/n}}X^\top X\Par{I - \exp\Par{-tX^\top X/n}}\Par{X^\top X}^+}\\
			& =  \tr\Par{ \Par{I - \exp\Par{-tX^\top X/n}}^2\Par{X^\top X}^+}\\
			& = \frac{1}{n} \sum_{i=1}^d \indi\Brace{\lambda_i \neq 0} \frac 1{\lambda_i} \Par{1-\exp(-t\lambda_i)}^2.
		\end{aligned}
	\end{equation}For the second term, 
	\begin{equation}
		\label{eq:computerisk5_under}
		\begin{aligned}
			& \E_{Z, P}\Brack{ \norm{ \Par{X^\top X}^+ \Par{I - \exp\Par{-tX^\top X/n}}X^\top{ Z(I-PP^\top)\theta^*} }_2^2} \\
			& =\E_{\Lambda}\Brack{ \frac{1}{n} \sum_{i=1}^d\indi\Brace{\lambda_i \neq 0} \frac 1{\lambda_i} \Par{1-\exp(-t\lambda_i)}^2 \cdot \frac{p-d}{p}\norm{\theta^*}_2^2},
		\end{aligned}
	\end{equation}
	which follows from $X$ and $Z(I-PP^\top)$ are independent and $Z(I-PP^\top)\theta^*$ is invariant under unitary transformation. Finally, we note that since $P$ is a semi-orthogonal matrix, the eigenvalues of $\frac 1 n X^\top X$ follows the distribution $\mathcal{D}_\Lambda$.
	Combining \eqref{eq:computerisk2_under}, \eqref{eq:computerisk3_under}, \eqref{eq:computerisk4_under} and \eqref{eq:computerisk5_under} completes the proof.
\end{proof}
Now, we are ready to give an approximation of the optimal stopping time $\topt$.
\underoptstopping*
\begin{proof}[Proof of Theorem~\ref{thm: under_optstopping}]
	In this proof, all expectations are taken over the distribution $\mathcal{D}_\Lambda$, so we omit them. 
	We first compute the derivative of $\oR_{\theta^*}$ with respect to time $t$. By Lemma~\ref{lem:underrisk}, 
	\begin{equation*}
		\begin{aligned}
			\frac d {dt} \oR_{\theta^*}(t) =\E\Brack{-\frac {2\norm{  \theta^*}_2^2} p \sum_{i=1}^d\lambda_i\exp\Par{-2t\lambda_i} + 
				\frac{2}{n}\Par{\sigma^2 + \frac{p-d}{p}\norm{\theta^*}_2^2}  \sum_{i=1}^d \Par{\exp(-t\lambda_i)-\exp(-2t\lambda_i)} }.
		\end{aligned}
	\end{equation*}
	We first derive an upper bound on $	\frac d {dt} \oR_{\theta^*}(t)$,
	\begin{equation*}
		\begin{aligned}
			\frac d {dt} \oR_{\theta^*}(t) & \leq \E\Brack{-\frac {2\norm{  \theta^*}_2^2} p \sum_{i=1}^d\Par{\lambda_i-  {2t\lambda_i^2} } + 
				\frac{2}{n}\Par{\sigma^2 + \frac{p-d}{p}\norm{\theta^*}_2^2}  \sum_{i=1}^d t\lambda_i } \\
			& = \E \Brack{-\frac{2\norm{\theta^*}_2^2}{p} \sum_{i=1}^d \lambda_i + \Par{\frac{4\norm{\theta^*}_2^2}{p}\sum_{i=1}^d \lambda_i^2 +\frac{2}{n}\Par{\sigma^2 + \frac{p-d}{p}\norm{\theta^*}_2^2}  \sum_{i=1}^d \lambda_i }t } . 
		\end{aligned}	
	\end{equation*}
	We used $\exp(-x)\geq 1-x$ and $\exp(-x) - \exp(-2x) \leq x$ for all $x\geq0$.
	Then, we have $	\frac d {dt} \oR_{\theta^*}(t) < 0$ at time $t <t_1$, where
	\begin{equation*}
		\begin{aligned} 
			t_1 = \frac{\frac{2\norm{\theta^*}_2^2}{p}\E\Brack{ \sum_{i=1}^d \lambda_i}}{\frac{4\norm{\theta^*}_2^2}{p}\E \Brack{\sum_{i=1}^d \lambda_i^2} +\frac{2}{n}\Par{\sigma^2 + \frac{p-d}{p}\norm{\theta^*}_2^2}  \E \Brack{\sum_{i=1}^d \lambda_i}} \\
			= \frac{{n\norm{\theta^*}_2^2}}{{2(n+d+2)\norm{\theta^*}_2^2} +{p\sigma^2 + \Par{p-d}\norm{\theta^*}_2^2} }
			\geq \frac{{n\norm{\theta^*}_2^2}}{ {p\sigma^2 +2\Par{p-d}\norm{\theta^*}_2^2} }.
		\end{aligned}
	\end{equation*}
	The second equation follows from Lemma~\ref{lem:eigenratio}. The last step follows from the assumption $(8n+9d+16)\norm{\theta^*}_2^2\leq{p\sigma^2 + p\norm{\theta^*}_2^2}$.
	Next, we derive a lower bound on $\frac d {dt} \oR_{\theta^*}(t) $ for $t < T = \frac{n}{4\Par{n+d+2}}$,
	\begin{equation*}
		\begin{aligned}
			\frac d {dt} \oR_{\theta^*}(t) & \geq \E\Brack{-\frac {2\norm{  \theta^*}_2^2} p \sum_{i=1}^d{\lambda_i } + 
				\frac{2}{n}\Par{\sigma^2 + \frac{p-d}{p}\norm{\theta^*}_2^2}  \sum_{i=1}^d \Par{t\lambda_i - 2tT\lambda_i^2}}.
		\end{aligned}
	\end{equation*}
	We used $\exp(-2x) \leq 1-2x+2x^2$ for all $x\geq0$.
	Then, we have $	\frac d {dt} \oR_{\theta^*}(t) \geq 0$ at time 
	\begin{equation*}
		\begin{aligned}
			t_2 & =  \frac{\frac{2\norm{\theta^*}_2^2}{p}\E\Brack{ \sum_{i=1}^d \lambda_i}}{\frac{2}{n}\Par{\sigma^2 + \frac{p-d}{p}\norm{\theta^*}_2^2}  \E\Brack{\sum_{i=1}^d \lambda_i - 2T\lambda_i^2}} \leq \frac {2{n\norm{\theta^*}_2^2}}{{p\sigma^2 + \Par{p-d}\norm{\theta^*}_2^2}  } 
		\end{aligned}
	\end{equation*}
	which follows from  Lemma~\ref{lem:eigenratio}. Under assumption $(8n+9d+16)\norm{\theta^*}_2^2\leq{p\sigma^2 + p\norm{\theta^*}_2^2}$, $t_2 < T$.
	Thus, we have $$  \frac{{n\norm{\theta^*}_2^2}}{ 2\Par{p\sigma^2 + \Par{p-d}\norm{\theta^*}_2^2} }\leq \topt\leq  \frac{2n{\norm{\theta^*}_2^2}}{{p\sigma^2 + \Par{p-d}\norm{\theta^*}_2^2}}$$
\end{proof}

\subsection{Proof of Proposition~\ref{thm:over_n_mono} and Proposition~\ref{thm:optrisk_under}}
\overnmono*
\begin{proof}
   By Lemma~\ref{lem:overrisk}, 
   \begin{equation*}
   	\begin{aligned}
    		\E_X \Brack{\oR_{X, \theta^*}(\tear)}  = 	\sigma^2 +\E_X \Brack{  \sum_{i=1}^d\exp(-2\tear\lambda_i) \frac{\norm{\theta^*}^2}{d} + \frac{\sigma^2}{n}\sum_{i=1}^d \indi\Brace{\lambda_i \neq 0} \frac 1 {\lambda_i} \Par{1-\exp(-\tear \lambda_i)}^2}.
   	\end{aligned}
   \end{equation*}
where $\lambda_i$'s are the eigenvalues of the matrix $\frac 1 n X^\top X$. Next, we can bound 
\begin{equation*}
	\begin{aligned}
		\E_X \Brack{   \frac{\sigma^2}{n}\sum_{i=1}^d \indi\Brace{\lambda_i \neq 0} \frac 1 {\lambda_i} \Par{1-\exp(-\tear \lambda_i)}^2} 
		\leq & ~\E_X \Brack{   \frac{\sigma^2}{n}\sum_{i=1}^d \indi\Brace{\lambda_i \neq 0} \frac 1 \lambda_i  \Par{1-\exp(-\tear \lambda_i) }}  \\
		\leq & ~\E_X \Brack{   \frac{\sigma^2}{n}\sum_{i=1}^d\indi\Brace{\lambda_i \neq 0} \frac 1 \lambda_i\cdot  \tear \lambda_i}  \\
		= & ~\frac {\alpha \sigma^2 \min\Brace{n,d}} {n+d}  \leq \frac 1 2 \alpha \sigma^2.
	\end{aligned}
\end{equation*}
We used $0\leq1-\exp(-x) \leq 1$ and $\exp(-x) \geq 1-x$ for $x \geq 0$.
It is straightforward to show $$\E_X \Brack{   \frac{\sigma^2}{n}\sum_{i=1}^d \frac 1 {\lambda_i} \Par{1-\exp(-\tear \lambda_i)}^2}  \geq 0.$$

Finally, we show that $\oR$ decreases as $n$ increases.  Let $X' \in \R^{(n+1) \times d}$ be a random matrix with the first $n$ rows the same as $X$ and the last row follow $\mathcal{N}(0,\id)$.
	Let $\lambda_1 \geq \lambda_2\geq ... \geq \lambda_d$ be the eigenvalues of the matrix $\frac 1 n X^\top X$ and $\lambda'_1 \geq \lambda'_2\geq ... \geq \lambda'_d$ be the eigenvalues of the matrix $\frac 1 n X'^\top X'$. Then, by Lemma~\ref{lem:interlacing}, we have $\lambda_i' \geq \lambda_i$ for all $i = 1,...,d$. The derivative 
	\begin{equation*}
		\frac d {d\lambda_i}  \Par{  \sum_{i=1}^d\exp(-2\tear\lambda_i) } =  -2\tear \exp\Par{-2\tear\lambda_i} \leq 0.
	\end{equation*}
	By coupling $X$ and $X'$, we have 
	\begin{equation*}
		\E_{X}\Brack{\sum_{i=1}^d \exp\Par{-2\frac{\alpha n}{n+d}\lambda_i} } \geq 		\E_{X'}\Brack{\sum_{i=1}^d \exp\Par{-2\frac{\alpha n}{n+d}\lambda_i'} } \geq \E_{X'}\Brack{\sum_{i=1}^d \exp\Par{-2\frac{\alpha (n+1)}{n+1+d}\lambda_i'} } ,
	\end{equation*}
	which shows $\oR$ decreases as $n$ increases.
\end{proof}

\optriskunder*
\begin{proof}
	By Lemma~\ref{lem:underrisk}, 
	\begin{equation*}
		\begin{aligned}
			\oR_{\theta^*}( \tear) = \sigma^2 + \Par{1-\frac d p} \norm{\theta^*}_2^2 + \E_{\Lambda\sim\mathcal{D}_\Lambda}\Brack{\sum_{i=1}^d \exp\Par{-2 \tear\lambda_i} \frac {\norm{  \theta^*}_2^2} p} + \\
			\E_{\Lambda\sim\mathcal{D}_\Lambda}\Brack{ \frac{1}{n} \sum_{i=1}^d \frac 1{\lambda_i} \Par{1-\exp(- \tear\lambda_i)}^2 \Par{\sigma^2 + \frac{p-d}{p}\norm{\theta^*}_2^2}}
		\end{aligned}
	\end{equation*}
	$\mathcal{D}_\Lambda$ is the distribution of eigenvalues of a matrix $\frac 1 n \tilde{X}^\top \tilde{X}$ where each entry of $\tilde {X}$ is i.i.d standard normal random variable. 
	We first derive an upper bound on $\oR_{\theta^*}( \tear) $. For $\tear = \frac{{n\norm{\theta^*}_2^2}}{ {p\sigma^2 + \Par{p-d}\norm{\theta^*}_2^2} }$,
	\begin{equation*}
		\begin{aligned}
			&\oR_{\theta^*}( \tear)\\
			\leq ~& \sigma^2 + \norm{\theta^*}_2^2 + \frac {\norm{  \theta^*}_2^2} p \E \Brack{\sum_{i=1}^d-2 \tear \lambda_i + 2\tear^2 \lambda_i^2}+\frac 1 n \Par{\sigma^2+ \frac{p-d}{p}\norm{\theta^*}_2^2} \E \Brack{\sum_{i=1}^d \tear^2\lambda_i}\\
			= ~&  \sigma^2 + \norm{\theta^*}_2^2 - \frac {2 \tear \norm{  \theta^*}_2^2d} p + \frac {2 \tear^2 \norm{  \theta^*}_2^2d(d+n+2)} {pn} + \frac d n \Par{\sigma^2+ \frac{p-d}{p}\norm{\theta^*}_2^2} \tear^2 \\
			= ~&  \sigma^2 + \norm{\theta^*}_2^2 - \frac { nd \norm{  \theta^*}_2^4}{p^2\sigma^2 + p\Par{p-d}\norm{\theta^*}_2^2}  +  \frac {2  \norm{  \theta^*}_2^6nd(d+n+2)} {p\Par{p\sigma^2 + \Par{p-d}\norm{\theta^*}_2^2}^2}.
		\end{aligned}
	\end{equation*}
	The inequality follows from $\exp(-x) \leq 1-x+\frac 1 2 x^2$ for $x\geq0$. The second step follows from Lemma~\ref{lem:eigenexpect}. When $8(n+d+2)\norm{\theta^*}_2^2\leq{p\sigma^2 + \Par{p-d}\norm{\theta^*}_2^2}$, we have 
	$$\oR_{\theta^*}( \tear) \leq \sigma^2 + \norm{\theta^*}_2^2 - \frac { 3nd \norm{  \theta^*}_2^4}{4p^2\sigma^2 + p\Par{p-d}\norm{\theta^*}_2^2} = \overline{R}_2.  $$
	Next, we derive a lower bound on $\oR_{\theta^*}(\alpha \tear) $.
	\begin{equation*}
		\begin{aligned}
			\oR_{\theta^*}( \tear)
			&\geq  \sigma^2 + \norm{\theta^*}_2^2 + \frac {\norm{  \theta^*}_2^2} p \E \Brack{\sum_{i=1}^d-2 \tear \lambda_i }
			=  \sigma^2 + \norm{\theta^*}_2^2 - \frac { 2nd \norm{  \theta^*}_2^4}{p^2\sigma^2 + p\Par{p-d}\norm{\theta^*}_2^2}  =\overline{R}_1  .
		\end{aligned}
	\end{equation*}
	The inequality follow from $\exp(-x) \geq 1-x$. The second step follows from Lemma~\ref{lem:eigenexpect}.
	Since $d \leq p$ and $(8n+9d+16)\norm{\theta^*}_2^2\leq{p\sigma^2 + p\norm{\theta^*}_2^2}$,  $ \frac { nd \norm{  \theta^*}_2^4}{p^2\sigma^2 + p\Par{p-d}\norm{\theta^*}_2^2} \leq \frac 1 8 \norm{\theta^*}_2^2$, which shows $$\frac{\overline{R}_1}{\overline{R}_2} \geq \frac{1-\frac{1}{8}\cdot 2 }{1-\frac 1 8 \cdot \frac 3 4} \geq 0.8.$$ Both $\overline{R}_1$ and $\overline{R}_2$ decreases when $d$ and $n$ increases.
\end{proof}

\section{Asymptotic Regime}

\label{sec:asymptotic}

In this section, we examine the behavior of the optimal stopping time in the asymptotic regime where $n, d\rightarrow \infty$. We consider the setting $\so$.
We define $\gamma = n/d$ and use $X_n$ to denote the data matrix with $n$ samples.

By Lemma~\ref{lem:overrisk}, the derivative of the expected risk is given by
\begin{align*}
	\frac d {dt} \oR_{X, \theta^*}(t) 
	&=  -\sum_{i=1}^d 2\lambda_i \exp\Par{-2t\lambda_i}\frac { \norm{\theta^*}_2^2}d + \frac{2\sigma^2}{n} \sum_{i=1}^d  \Par{\exp(-t\lambda_i)-\exp(-2t\lambda_i)}%
\end{align*}
$\lambda_{n+1} = ...= \lambda _d = 0$ if $d > n$.

Using the Marchenko–Pastur distribution, we can take the limit of $d, n\rightarrow \infty$ and get
\begin{equation}
\begin{aligned}
	\frac d {dt} \oR_{X_{\infty}, \theta^*}(t) 
	&= \int_{(\sqrt{1/\gamma}-1)^2}^{(\sqrt{1/\gamma}+1)^2} \left( - {2 \norm{\theta^*}_2^2} \lambda \exp\Par{-2t\lambda} + \frac 2 \gamma \sigma^2 \Par{\exp(-t\lambda)-\exp(-2t\lambda)} \right) d \mu(\lambda) \nonumber\\
	&= \int_{(\sqrt{1/\gamma}-1)^2}^{(\sqrt{1/\gamma}+1)^2} \left( - {2 \norm{\theta^*}_2^2} \lambda \exp\Par{-2t\lambda} + \frac 2 \gamma \sigma^2 \Par{\exp(-t\lambda)-\exp(-2t\lambda)} \right) d F(\lambda), \label{eq:M_P_dist}
\end{aligned}
\end{equation}
where 
\[
d\mu(\lambda) = \ind\left\{\lambda\in[(\sqrt{1/\gamma}-1)^2, (\sqrt{1/\gamma}+1)^2]\right\} d F(\lambda) + \max\left\{ 0, 1 - \gamma \right\} \ind\left\{\lambda=0\right\}d\lambda,
\]
and 
\[
d F(\lambda) = \frac{\gamma}{2\pi} \frac{ \sqrt{ ( (\sqrt{1/\gamma}+1)^2 - \lambda )( \lambda - (\sqrt{1/\gamma}-1)^2 ) } }{\lambda} d\lambda.
\]
The second step follows from $- {2 \norm{\theta^*}_2^2\gamma} \lambda \exp\Par{-2t\lambda} + 2\sigma^2 \Par{\exp(-t\lambda)-\exp(-2t\lambda)} = 0$ when $\lambda=0$.

We first consider the case where $1/\gamma\geq4$ ($d\geq4n$).
Then on the support of $\mu(\lambda)$, $\lambda\in[\frac{1}{4\gamma} , \frac{9}{4\gamma}  ]$.
Therefore, when $t<\frac{4}{9}{\gamma}\log\left( 1 + \frac{\left\|\theta\right\|_2^2}{4\sigma^2} \right)$, $\frac d {dt} \oR_{X_{\infty}, \theta^*}(t) > 0$; when $t>{4\gamma}\log\Par{1 + \frac{9\norm{\theta}_2^2}{4\sigma^2 }}$, $\frac d {dt} \oR_{X_{\infty}, \theta^*}(t) < 0$.

We then consider the case where $\gamma\geq4$ ($d<\frac{n}{4}$).
Then on the support of $\mu(\lambda)$, $\lambda\in[\frac{1}{4}, \frac{9}{4} ]$.
Therefore, when $t<\frac{4}{9} \log\left( 1 + \frac{\left\|\theta\right\|_2^2\gamma}{4\sigma^2 } \right)$, $\frac d {dt} \oR_{X_{\infty}, \theta^*}(t) > 0$; when $t > 4 \log\Par{1 + \frac{9\norm{\theta}_2^2\gamma }{4\sigma^2 }}$, $\frac d {dt} \oR_{X_{\infty}, \theta^*}(t) < 0$.

We now study the intermediate regime where $\frac{1}{4}<\gamma<4$.
We need the following lemma.
\begin{restatable}{lemma}{factexp}
	\label{fact:exp}
	For any $t\lambda\geq0$,
	let
	\begin{align*}
		G_1(t,\lambda) &= - {2 \norm{\theta^*}_2^2\gamma} (1 - t\lambda) + 2\sigma^2 \lrp{ t - \frac{3}{2} t^2 \lambda }, \\
		G_2(t,\lambda) &= -{2 \norm{\theta^*}_2^2\gamma} + 2\sigma^2 \lrp{ t - \frac{1}{2} t^2 \lambda }, \\
		G_3(t,\lambda) &= - {2 \norm{\theta^*}_2^2\gamma} \lrp{1-t\lambda} + 2 \sigma^2 t; \\ 
		\Gamma_1(t) &= 1-t(1-\sqrt{1/\gamma})^2, \\
		\Gamma_2(t) &= 1-t(1+\sqrt{1/\gamma})^2.
	\end{align*}
	Then,
	\begin{multline}
		- {2 \norm{\theta^*}_2^2\gamma} \exp\Par{-2t\lambda} + 2\sigma^2 \frac{\exp(-t\lambda)-\exp(-2t\lambda)}{\lambda} \\
		\geq \min\left\{ G_1(t,\lambda), G_2(t,\lambda), \Gamma_1(t) G_3(t,\lambda), \Gamma_2(t) G_3(t,\lambda) \right\},
		\label{eq:exp_lower_bound}
	\end{multline}
	and 
	\begin{multline}
		- {2 \norm{\theta^*}_2^2\gamma} \exp\Par{-2t\lambda} + 2\sigma^2 \frac{\exp(-t\lambda)-\exp(-2t\lambda)}{\lambda} \\
		\leq \max\left\{ G_3(t,\lambda), \Gamma_1(t) G_2(t,\lambda), \Gamma_2(t) G_2(t,\lambda), \Gamma_1(t) G_3(t,\lambda), \Gamma_2(t) G_3(t,\lambda) \right\}.
		\label{eq:exp_upper_bound}
	\end{multline}
\end{restatable}

\begin{restatable}{theorem}{asymptotic}
	\label{thm:asymptotic}
When $\frac{1}{4}<\gamma<4$ and when $\rho = \Par{\gamma+1} \frac{\lrn{\theta^*}_2^2}{\sigma^2} < 2 - \sqrt{3}$,
the risk $\oR_{X_{\infty}, \theta^*}(t)$ is decreasing in the interval
\begin{align}
t < \frac{\gamma}{\gamma+1} \cdot \frac{\rho}{\rho + 1},
\end{align}
and is increasing in the interval
\begin{align}
\frac{2\gamma}{\gamma+1} \cdot \frac{\rho}{ 1+\rho +\sqrt{(1+\rho )^2-6\rho } } 
< t 
< \frac{1}{\lrp{1+\sqrt{1/\gamma}}^2}.
\end{align}
\end{restatable}
Therefore, the optimal early stopping time $\topt $ satisfy:
\[
\frac{\gamma}{\gamma+1} \frac{\rho}{\rho + 1}
\leq \topt 
\leq \frac{2\gamma}{\gamma+1} \frac{\rho}{ 1+\rho +\sqrt{(1+\rho )^2-6\rho } }.
\]

Theorem~\ref{thm:asymptotic} shows that when $\frac 1 4<\gamma <4$, $\topt = \Theta\Par{\frac{\gamma}{\gamma+1}} = \Theta\Par{\frac n {n+d}}$ still holds, which completes the missing piece of Theorem~\ref{thm: over_optstopping}.

\begin{proof}
By \eqref{eq:M_P_dist}, it suffices to find the positive and negative parts of 
\begin{align*}
	&\int_{(\sqrt{1/\gamma}-1)^2}^{(\sqrt{1/\gamma}+1)^2} \left( - {2 \norm{\theta^*}_2^2\gamma} \lambda \exp\Par{-2t\lambda} + 2\sigma^2 \Par{\exp(-t\lambda)-\exp(-2t\lambda)} \right) d F(\lambda) \\
	&= \int_{(\sqrt{1/\gamma}-1)^2}^{(\sqrt{1/\gamma}+1)^2} \left( - {2 \norm{\theta^*}_2^2\gamma} \exp\Par{-2t\lambda} + 2\sigma^2 \frac{\exp(-t\lambda)-\exp(-2t\lambda)}{\lambda} \right) \lambda d F(\lambda),
\end{align*}
where $\lambda d F(\lambda) = \frac{\gamma}{2\pi} { \sqrt{ ( (\sqrt{\gamma}+1)^2 - \lambda )( \lambda - (\sqrt{\gamma}-1)^2 ) } } d\lambda$.

We note that 
\begin{align*}
	\int_{(\sqrt{1/\gamma}-1)^2}^{(\sqrt{1/\gamma}+1)^2} \lambda d F(\lambda)
	= \lim_{d, n\rightarrow\infty} \frac{1}{d} \mathrm{Tr} \lrp{ \E{ \left[ \frac 1 n X_n^\top X_n \right] } } 
	= 1,
\end{align*}
and that
\begin{align*}
	\int_{(\sqrt{1/\gamma}-1)^2}^{(\sqrt{1/\gamma}+1)^2} \lambda^2 d F(\lambda)
	= \lim_{d,n\rightarrow\infty} \frac{1}{d} \E{ \lrn{\frac 1 n \lrp{X_n^\top X_n}} }_F^2
	= 1 +\frac 1 \gamma.
\end{align*}

Applying these results on $G_1(t,\lambda)$, $G_2(t,\lambda)$, and $G_3(t,\lambda)$ (defined in Lemma~\ref{fact:exp}), we obtain that 
\begin{align*}
	E_1(t) = \int_{(\sqrt{1/\gamma}-1)^2}^{(\sqrt{1/\gamma}+1)^2} G_1(t,\lambda) \lambda d F(\lambda) 
	&= - {2 \norm{\theta^*}_2^2\gamma} (1 - t (1+1/\gamma) ) + 2\sigma^2 \lrp{ t - \frac{3}{2} t^2 (1+1/\gamma) }, \\
	E_2(t) = \int_{(\sqrt{1/\gamma}-1)^2}^{(\sqrt{1/\gamma}+1)^2} G_2(t,\lambda) \lambda d F(\lambda) 
	&= - {2 \norm{\theta^*}_2^2\gamma} + 2\sigma^2 \lrp{ t - \frac{1}{2} t^2 (1+1/\gamma) }, \\
	E_3(t) = \int_{(\sqrt{1/\gamma}-1)^2}^{(\sqrt{1/\gamma}+1)^2} G_3(t,\lambda) \lambda d F(\lambda) 
	&= - {2 \norm{\theta^*}_2^2\gamma} (1 - t (1+1/\gamma) ) + 2\sigma^2 t.
\end{align*}

From inequalities~\eqref{eq:exp_lower_bound} and~\eqref{eq:exp_upper_bound} in Lemma~\ref{fact:exp}, we know that $\frac d {dt} \oR_{X_{\infty}, \theta^*}(t) > 0$ if 
$$\min\left\{ E_1(t), E_2(t), \Gamma_1(t) E_3(t), \Gamma_2(t) E_3(t) \right\} > 0,$$ 
and $\frac d {dt} \oR_{X_{\infty}, \theta^*}(t) < 0$ if
$$\max\left\{ E_3(t), \Gamma_1(t) E_2(t), \Gamma_2(t) E_2(t), \Gamma_1(t) E_3(t), \Gamma_2(t) E_3(t) \right\} < 0.$$

\paragraph{Increasing interval:}
We can develop this condition further to be: $\frac d {dt} \oR_{X_{\infty}, \theta^*}(t) > 0$ if the following event happens:
\begin{align*}
	\{ E_1(t) > 0 \} \bigcap \{ E_2(t) > 0 \} \bigcap \bigg( &\lrp{ \{ \Gamma_1(t) >0 \} \bigcap \{\Gamma_2(t)>0\} \bigcap \{E_3(t)>0\} } \\
	&\bigcup \lrp{ \{ \Gamma_1(t) < 0 \} \bigcap \{\Gamma_2(t)<0\} \bigcap \{E_3(t)<0\} } \bigg).
\end{align*}

	$E_1(t) > 0$ is equivalent to the condition that $1+{\gamma} > (2+\sqrt{3})\frac{\sigma^2}{\lrn{\theta^*}_2^2}$ or $1+{\gamma} < (2-\sqrt{3})\frac{\sigma^2}{\lrn{\theta^*}_2^2}$, and that 
	\begin{multline*}
		\frac{ 1 + \frac{\lrn{\theta^*}_2^2}{\sigma^2}\Par{\gamma+1} - \sqrt{\lrp{1 + \frac{\lrn{\theta^*}_2^2}{\sigma^2}\Par{\gamma+1}}^2 - 6\frac{\lrn{\theta^*}_2^2}{\sigma^2}\Par{\gamma+1} } }{3(1/\gamma+1)} \\
		< t 
		< 	\frac{ 1 + \frac{\lrn{\theta^*}_2^2}{\sigma^2}\Par{\gamma+1} + \sqrt{\lrp{1 + \frac{\lrn{\theta^*}_2^2}{\sigma^2}\Par{\gamma+1}}^2 - 6\frac{\lrn{\theta^*}_2^2}{\sigma^2}\Par{\gamma+1} } }{3(1/\gamma+1)}.
	\end{multline*}

	$E_2(t) > 0$ is equivalent to the condition that $1+{\gamma} < \frac{1}{2} \frac{\sigma^2}{\lrn{\theta^*}_2^2}$ and that 
	\[
	\frac{ \gamma - \gamma\sqrt{ 1-2\Par{\gamma+1} \frac{\lrn{\theta^*}_2^2}{\sigma^2} } }{\gamma+1} < 
	t < \frac{ \gamma + \gamma\sqrt{ 1-2\Par{\gamma+1} \frac{\lrn{\theta^*}_2^2}{\sigma^2} } }{\gamma+1}.
	\]
	Then the event $\{ E_1(t) > 0 \} \bigcap \{ E_2(t) > 0 \}$ leads to $\rho<2-\sqrt{3}$ and that $t$ needs to satisfy the following condition for both events to happen:
	\[
	\frac{2\gamma}{\gamma+1} \cdot \frac{\rho}{ 1+\rho +\sqrt{(1+\rho )^2-6\rho } } 
	< t 
	< \frac{\gamma}{3(\gamma+1)} \lrp{1+\rho+\sqrt{(1+\rho)^2-6\rho}}.
	\]

	Since 
	\[
	t > \frac{2\gamma}{\gamma+1} \cdot \frac{\rho}{ 1+\rho +\sqrt{(1+\rho )^2-6\rho } } 
	> \frac{\gamma}{\gamma+1} \cdot \frac{\rho}{\rho+1}
	= \frac{{\gamma}}{\rho+1}  \frac{\lrn{\theta^*}_2^2}{\sigma^2},
	\]
	$E_3(t) > 0$.
	Then $\{ \Gamma_1(t) >0 \} \bigcap \{\Gamma_2(t)>0\}$ leads to $t < \frac{1}{\lrp{1+\sqrt{1/\gamma}}^2}$.

Combining the intervals, we obtain the sufficient condition for $\frac d {dt} \oR_{X_{\infty}, \theta^*}(t) > 0$ to be $\rho = \frac{\gamma+1}{\gamma} \frac{\lrn{\theta^*}_2^2}{\sigma^2} < 2 - \sqrt{3}$ and
\begin{align}
	\frac{2\gamma}{\gamma+1} \cdot \frac{\rho}{ 1+\rho +\sqrt{(1+\rho )^2-6\rho } } 
	< t 
	< \frac{1}{\lrp{1+\sqrt{1/\gamma}}^2}.
	\label{eq:increasing_interval}
\end{align}

\paragraph{Decreasing interval:}
We assume that $\rho < 2 - \sqrt{3}$.

	$E_3(t)<0$ leads to
	\[
	t < \frac{\gamma}{\gamma+1} \cdot \frac{\rho}{\rho+1}.
	\]
	Under this condition, $\Gamma_1(t) > \Gamma_2(t) > 0$.
	We therefore need $E_2(t)<0$ and $E_3(t)<0$.

	$E_2(t) < 0$ is equivalent to 
	\[
	t < \frac{\gamma}{\gamma+1} \cdot \frac{2\rho}{1+\sqrt{1-2\rho}} \qquad
	\text{or} \qquad
	t > \frac{\gamma}{\gamma+1} \lrp{1+\sqrt{1-2\rho}}.
	\]

	$E_1(t) < 0$ is equivalent to
	\[
	t < \frac{2\gamma}{\gamma+1} \cdot \frac{\rho}{ 1+\rho +\sqrt{(1+\rho )^2-6\rho } }  \qquad
	\text{or} \qquad
	t > \frac{\gamma}{3(\gamma+1)} \lrp{1+\rho+\sqrt{(1+\rho)^2-6\rho}}.
	\]

Joining the intervals under the condition that $\rho<2-\sqrt{3}$, we obtain that the sufficient condition for $\frac d {dt} \oR_{X_{\infty}, \theta^*}(t) < 0$ is
\begin{align}
	t < \frac{\gamma}{\gamma+1} \cdot \frac{\rho}{\rho+1}.
	\label{eq:decreasing_interval}
\end{align}
\end{proof}

Now, we prove Lemma~\ref{fact:exp}.
\begin{proof}[Proof of Lemma~\ref{fact:exp}]
We first state an expansion result about the exponential function:
\begin{align} 1 - t\lambda + \frac{1}{2} t^2\lambda^2 \leq \exp\lrp{-t\lambda} &\leq 1 - t\lambda \exp\lrp{-t\lambda}. \nonumber
\end{align}
We decompose our argument into two terms:

\begin{align*}
 & - {2 \norm{\theta^*}_2^2\gamma} \exp\Par{-2t\lambda} + 2\sigma^2 \frac{\exp(-t\lambda)-\exp(-2t\lambda)}{\lambda} \\
= & \exp(-t\lambda) \cdot \lrp{ - {2 \norm{\theta^*}_2^2\gamma} \exp\Par{-t\lambda} + 2\sigma^2 \frac{1-\exp(-t\lambda)}{\lambda} }.
\end{align*}

We can bound the above terms by
\begin{equation}
    1-t\lambda \leq \exp (-t\lambda )\leq 1,
    \label{eq:exp1}
\end{equation}
and
\begin{equation}
\begin{aligned}
- {2 \norm{\theta^*}_2^2\gamma} + 2\sigma^2 \lrp{ t - \frac{1}{2} t^2 \lambda } 
& \leq - {2 \norm{\theta^*}_2^2\gamma} \exp\Par{-t\lambda} + 2\sigma^2 \frac{1-\exp(-t\lambda)}{\lambda} \\
& \leq - {2 \norm{\theta^*}_2^2\gamma} \lrp{1-t\lambda} + 2 \sigma^2 t. 
\end{aligned}
\label{eq:exp2}
\end{equation}
Then, combining \eqref{eq:exp1} and \eqref{eq:exp2}, we have
\begin{align*}
& - 2 \norm{\theta^*}_2^2\gamma \exp\Par{-2t\lambda} + 2\sigma^2 \frac{\exp(-t\lambda)-\exp(-2t\lambda)}{\lambda} \\
\geq &  \min\Bigg\{ (1 - t\lambda)\lrp{-2 \norm{\theta^*}_2^2\gamma + 2\sigma^2 \lrp{ t - \frac{1}{2} t^2 \lambda }}, - 2 \norm{\theta^*}_2^2\gamma+ 2\sigma^2 \lrp{ t - \frac{1}{2} t^2 \lambda }, \\
&(1-t\lambda)\lrp{ - 2 \norm{\theta^*}_2^2\gamma \lrp{1-t\lambda} + 2 \sigma^2 t } \Bigg\} \\
\geq & \min\Bigg\{ - 2 \norm{\theta^*}_2^2\gamma (1 - t\lambda) + 2\sigma^2 \lrp{ t - \frac{3}{2} t^2 \lambda }, 
 - 2 \norm{\theta^*}_2^2\gamma + 2\sigma^2 \lrp{ t - \frac{1}{2} t^2 \lambda }, \\
&\lrp{1-t(1+\sqrt{\gamma})^2} \lrp{ -2 \norm{\theta^*}_2^2\gamma \lrp{1-t\lambda} + 2 \sigma^2 t }, 
 \lrp{1-t(1-\sqrt{\gamma})^2} \lrp{ -2 \norm{\theta^*}_2^2\gamma \lrp{1-t\lambda} + 2 \sigma^2 t } \Bigg\},
\end{align*}
where we have used the fact that $\lambda\in[(1-\sqrt{1/\gamma})^2, (1+\sqrt{1/\gamma})^2]$.
On the flip side, we obtain that 
\begin{equation*}
\begin{aligned}
- {2 \norm{\theta^*}_2^2\gamma} \exp\Par{-2t\lambda} + 2\sigma^2 \frac{\exp(-t\lambda)-\exp(-2t\lambda)}{\lambda} \\
\leq \max\Bigg\{ - 2 \norm{\theta^*}_2^2\gamma \lrp{1-t\lambda} + 2 \sigma^2 t,
&\ (1 - t\lambda)\lrp{- 2 \norm{\theta^*}_2^2\gamma + 2\sigma^2 \lrp{ t - \frac{1}{2} t^2 \lambda }}, \\
&\ (1 - t\lambda)\lrp{ - 2 \norm{\theta^*}_2^2\gamma \lrp{1-t\lambda} + 2 \sigma^2 t } \Bigg\} \\
\leq \max\Bigg\{ - 2 \norm{\theta^*}_2^2\gamma \lrp{1-t\lambda} + 2 \sigma^2 t,
& \lrp{1-t(1+\sqrt{1/\gamma})^2}\lrp{- 2 \norm{\theta^*}_2^2\gamma + 2\sigma^2 \lrp{ t - \frac{1}{2} t^2 \lambda }}, \\
& \lrp{1-t(1-\sqrt{1/\gamma})^2}\lrp{- 2 \norm{\theta^*}_2^2\gamma + 2\sigma^2 \lrp{ t - \frac{1}{2} t^2 \lambda }},\\
& \lrp{1-t(1+\sqrt{1/\gamma})^2}\lrp{ - 2 \norm{\theta^*}_2^2\gamma \lrp{1-t\lambda} + 2 \sigma^2 t }, \\
& \lrp{1-t(1-\sqrt{1/\gamma})^2}\lrp{ - 2 \norm{\theta^*}_2^2\gamma \lrp{1-t\lambda} + 2 \sigma^2 t }.
\Bigg\}
\end{aligned}
\end{equation*}
For
\begin{align*}
G_1(t,\lambda) &= - 2 \norm{\theta^*}_2^2\gamma (1 - t\lambda) + 2\sigma^2 \lrp{ t - \frac{3}{2} t^2 \lambda }, \\
G_2(t,\lambda) &= -2 \norm{\theta^*}_2^2\gamma + 2\sigma^2 \lrp{ t - \frac{1}{2} t^2 \lambda }, \\
G_3(t,\lambda) &= -2 \norm{\theta^*}_2^2\gamma \lrp{1-t\lambda} + 2 \sigma^2 t; \\ 
\Gamma_1(t) &= 1-t(1-\sqrt{1/\gamma})^2, \\
\Gamma_2(t) &= 1-t(1+\sqrt{1/\gamma})^2,
\end{align*}
we have
\begin{multline*}
- 2 \norm{\theta^*}_2^2\gamma \exp\Par{-2t\lambda} + 2\sigma^2 \frac{\exp(-t\lambda)-\exp(-2t\lambda)}{\lambda} \\
\geq \min\left\{ G_1(t,\lambda), G_2(t,\lambda), \Gamma_1(t) G_3(t,\lambda), \Gamma_2(t) G_3(t,\lambda) \right\},
\end{multline*}
and 
\begin{multline*}
-2 \norm{\theta^*}_2^2\gamma\exp\Par{-2t\lambda} + 2\sigma^2 \frac{\exp(-t\lambda)-\exp(-2t\lambda)}{\lambda} \\
\leq \max\left\{ G_3(t,\lambda), \Gamma_1(t) G_2(t,\lambda), \Gamma_2(t) G_2(t,\lambda), \Gamma_1(t) G_3(t,\lambda), \Gamma_2(t) G_3(t,\lambda) \right\}.
\end{multline*}
\end{proof} %

\section{Technical Lemmas}
\begin{lemma}[\citep{davidson2001local}]
	\label{lem:eig_concentration}
	Let $A$ be a $m \times N$ matrix with $N \geq m$  whose entries are real, independent Gaussian random variables following $\mathcal{N}(0,1)$. Let $\sigma_1(A) \geq \sigma_2(A) \geq ... \geq \sigma_m(A)$ be the singular values of $A$.
	Then,
	\begin{equation*}
		\begin{aligned}
			\Pr\Brack{  \sqrt{N} - \sqrt{m} - t \leq \sigma_m\Par{A}  \leq \sigma_1(A) \leq  \sqrt{N} + \sqrt{m} +t  } \leq 1- 2\exp(-t^2/2).
		\end{aligned}
	\end{equation*}
\end{lemma}

\begin{lemma}
	\label{lem:eigenexpect}
	Let $A$ be a $m \times N$ matrix whose entries are real, independent Gaussian random variables following $\mathcal{N}(0,1)$. Let $\lambda_1(A) \geq \lambda_2(A) \geq ... \geq \lambda_N(A)$ be the eigenvalues of $\frac 1 m A^\top A$. Then, 
	\label{lem:eigenratio}
	\[
	\E \Brack{\sum_{i=1}^{N}\lambda_i} = N
	\]
	and
	\[
	\E \Brack{\sum_{i=1}^{N}\lambda_i^2} = \frac N m(m + N + 2).
	\]
\end{lemma}
\begin{proof}
	We first note that
	\begin{align*}
		m\E\Brack{\sum_{i=1}^N\lambda_i} = \E\Brack{\mathrm{Tr} \lrp{ {A}^\top {A} } } 
		= \E\Brack{ \lrn{{A}}_F^2 }
		= mN.
	\end{align*}
	We then note that 
	\begin{align*}
		m^2\E\Brack{\sum_{i=1}^N \lambda_i^2} = \E\Brack{\mathrm{Tr} \lrp{ \lrp{{A}^\top {A}}^2 } } 
		= \E\Brack{ \lrn{{A}^\top {A}}_F^2 },
	\end{align*}
	where 
	\[
	\lrp{ {A}^\top {A} }_{i,j} = \sum_{k=1}^m {A}_{k,i} {A}_{k,j}.
	\]
	Therefore, 
	\begin{align*}
		\E\Brack{ \lrn{{A}^\top {A}}_F^2 }
		= \sum_{i=1}^N \sum_{j=1}^N \E\Brack{ \lrp{\sum_{k=1}^m {A}_{k,i} {A}_{k,j}}^2 }
		= \sum_{i=1}^N \sum_{j=1}^N \sum_{k=1}^m \sum_{l=1}^m \E\Brack{{A}_{k,i} {A}_{k,j} {A}_{l,i} {A}_{l,j} }.
	\end{align*}
	Since all the odd moments of $x$ are zero, the nonzero elements in the above equation appear when $k=l$ or when $i=j$.
	Hence the number of nonzero elements is $mN(m+N)$.
	When either $k=l$ or $i=j$ happens, the second order moment of $\widetilde{a}$ is $1$.
	When both $k=l$ and $i=j$ happen, the fourth order moment of $\widetilde{a}$ is $3$.
	Therefore, 
	\begin{align*}
		\E\Brack{\sum_{i=1}^N \lambda_i^2}
		= \frac 1 {m^2}\sum_{i=1}^N \sum_{j=1}^N \sum_{k=1}^m \sum_{l=1}^m \E\Brack{ {A}_{k,i} {A}_{k,j} {A}_{l,i} {A}_{l,j} }
		= \frac N m(m+N+2).
	\end{align*}
\end{proof}

\begin{lemma}[Cauchy Interlacing Theorem (Corollary 4.3.9, \cite{horn2012matrix})]
\label{lem:interlacing}
Let $A \in \R^{m\times m}$ be a symetric matrix. Let $\lambda_1 \geq \lambda_2 \geq ...\geq \lambda_n$ be the eigenvalues of the matrix $A$ and let $z\in \R^m$ be a nonzero vector. Then, 
\begin{equation*}
	\label{lem:interlacing}
	\begin{aligned}
	\lambda_i(A) &\leq \lambda_i (A + zz^\top) \leq \lambda_{i-1}(A), \;\; i = 2,...,n\\
	\lambda_1(A) & \leq \lambda_1(A + zz^\top).
	\end{aligned}
\end{equation*}
\end{lemma}
\section{Discretization errors}
\label{sec:discrete}
In practice, one would discretize the algorithm and perform gradient descent, instead of using the gradient flow, to optimize over empirical loss.
This can be formulated in the following expression:
\begin{align*}
\beta_{k+1} = \beta_k - \frac{h}{n} X^\top X \beta_k + \frac{h}{n} X^\top y,
\end{align*}
where $h$ is the step size of the algorithm.
A solution to the above iterative equation, starting at $\beta_0$ is:
\begin{align}
\beta_k &= \lrp{I-\frac{h}{n}X^\top X}^k \beta_0 + \frac{h}{n} \lrp{ \sum_{i=0}^{k-1} \lrp{I-\frac{h}{n}X^\top X}^i } X^\top y \notag\\
&= \lrp{I-\frac{h}{n}X^\top X}^k \beta_0 + \lrp{X^\top X}^\dag \lrp{I-\lrp{I-\frac{h}{n}X^\top X}^k } X^\top y.\notag
\end{align}
Taking $\beta_0=0$, we obtain that for any $h<1/s_{\max}$, where $s_{\max}$ is the largest eigenvalue of $X^\top X/n$,
\begin{align}
\sup_{k} \lrn{ \beta_k - \hat{\beta}(kh) } &= \lrp{X^\top X}^\dag \lrp{ \exp\lrp{- \frac{kh}{n} X^\top X } - \lrp{I-\frac{h}{n}X^\top X}^k } X^\top y \notag\\
&\leq \frac{h}{2n} \lrn{X^\top y} \lrp{\exp(h s_{\max}) - 1}, \notag
\end{align}
since $\sup_{\tau>0} | e^{-l\tau}-(1-l)^\tau | \leq \frac{l}{2} \lrp{ e^l - 1 }$, $\forall l < 1$.

\section{Additional details of experiments}
\label{sec:exp_appendix}
\subsection{Linear regression}
\label{sec:exp_linear}
\begin{figure}[h!]
\centering

\includegraphics[width=\linewidth]{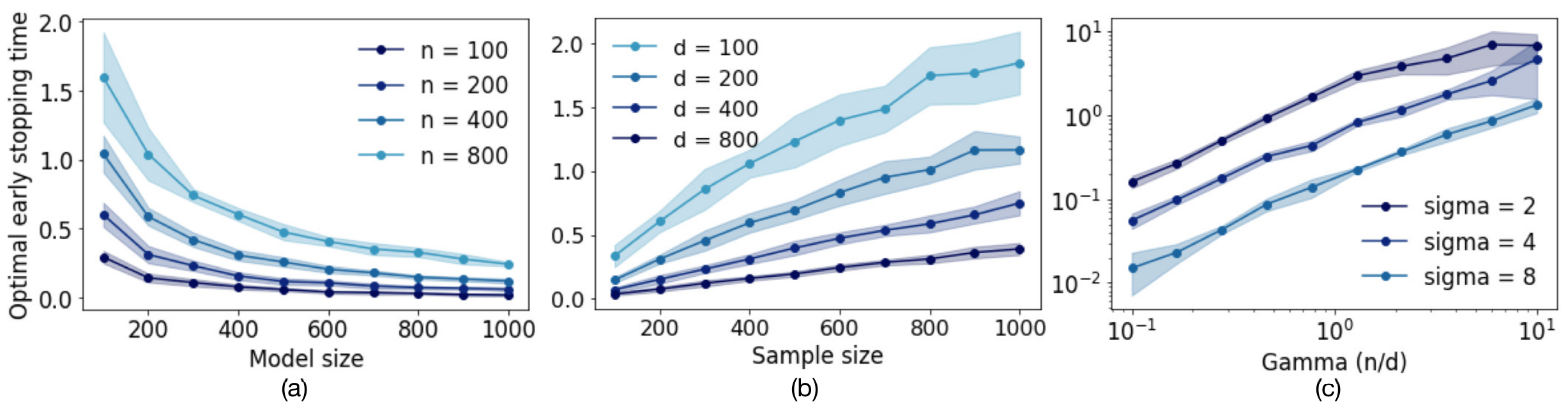}

\caption{(a) Optimal early stopping time decreases with increasing model sizes. (b) Optimal early stopping time increases with increasing sample sizes. (c) Optimal early stopping time increases with increasing $\gamma$. Shaded areas give the empirical bounds generated by 10 runs of the experiments.} 

\label{fig:overParam_n_and_d_CI}
\end{figure}

\paragraph{Optimal early stopping time in the over-informative parametrization setting.} In this experiment (Figure \ref{fig:overParam_n_and_d_CI}), we follow the setting $\so$. We use $p =10$ for all $n$ and $d$. The optimal early stopping time is selected as the time with the first local minimal loss.
We observe from these three plots that the optimal early stopping time decreases with larger model size $d$, and increases with more samples $n$. 

\begin{figure}[h!]
\centering

\includegraphics[width=0.7\linewidth]{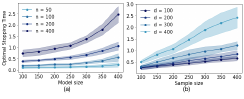}

\caption{(a) Optimal early stopping time increases with increasing model sizes. (b) Optimal early stopping time increases with increasing sample sizes. Shaded areas give the empirical bounds generated by 10 runs of the experiments.}
\label{fig:underParam_n_and_d_CI}

\end{figure}

\paragraph{Optimal early stopping time in the under-informative parametrization setting.} In this experiment (Figure \ref{fig:underParam_n_and_d_CI}), we follow the setting $\su$. We use $p = 400$ for all $n$ and $d$. The optimal early stopping time increases with increasing model sizes $d$, and increasing sample sizes $n$.

\paragraph{Risk monotonicity of optimal early stopping. } In Figure \ref{fig:risk_mono}, for over-informative models, we observe that the optimal early stopping risk decreases with increasing sample sizes $n$. For under-informative models, the optimal early stopping risk decreases with both increasing $d$ and $n$. This is consistent with the risk monotonicity analysis. 

\begin{figure}[h!]
\centering
\includegraphics[width=\linewidth]{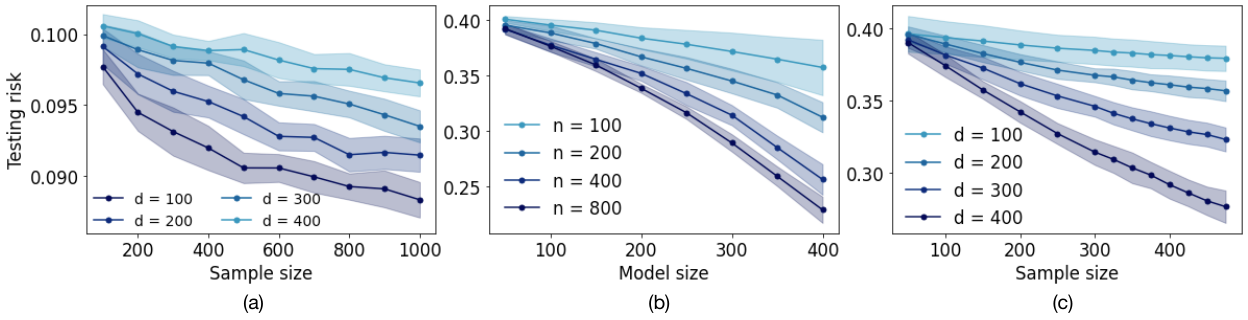}
\caption{(a) Testing risk at optimal early stopping time decreases with sample sizes in the over-informative parametrization setting. (b) (c) Testing risk at optimal early stopping time decreases with model sizes and sample sizes in the under-informative parametrization setting. }
\label{fig:risk_mono}
\end{figure}

\label{expr:settings}

\paragraph{Testing loss vs time.}
Figure~\ref{fig:visual_over} and Figure~\ref{fig:visual_under}
plot the testing loss as a function of time in the over- and under- informative parametrization settings respectively. The solid dots on the curves represent the points with lowest testing loss. We choose $p=10, \sigma=6$ and $p=500, \sigma=2$ in Figure~\ref{fig:visual_over} and Figure~\ref{fig:visual_under}. We can observe from Figure~\ref{fig:visual_over} the optimal early stopping time decreases with increasing $d$ and increases with increasing $n$. Moreover, testing loss decreases with increasing $n$. We can observe from Figure~\ref{fig:visual_under} the optimal early stopping time increases with both $d$ and $n$. 

\begin{figure}[h!]
\centering
\includegraphics[width=0.7\linewidth]{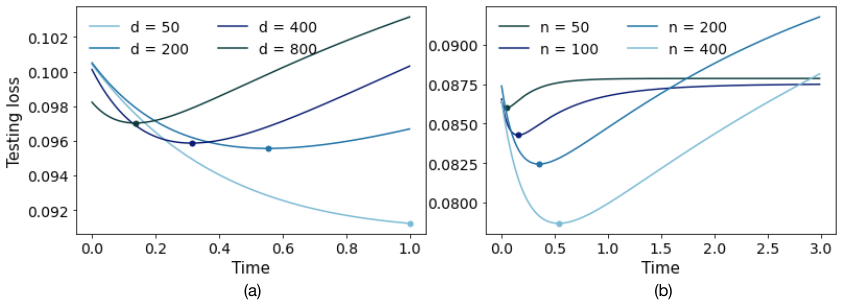}

\caption{(a) Testing loss behaviour for increasing training time with various model sizes. (b) Testing loss behaviour for increasing training time with various number of samples. }
\label{fig:visual_over}
\end{figure}

\begin{figure}[h!]
\centering
\includegraphics[width=0.7\linewidth]{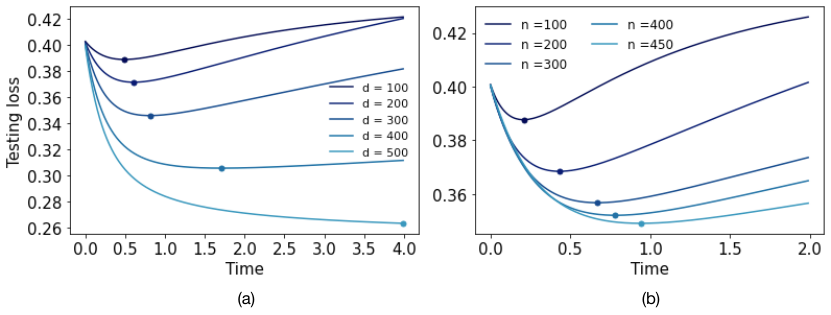}

\caption{(a) Testing loss behaviour for increasing training time with various model sizes. (b) Testing loss behaviour for increasing training time with various sample sizes. }
\label{fig:visual_under}
\end{figure}

\subsection{Additional plots image classification using deep neural networks}

\paragraph{Additional plots for varying widths.} In Figure~\ref{fig:error_test}, we show the testing error as a function of epoch for training MNIST, CIFAR-10 and CIFAR-100 using models of different width. We can observe the double descent phenomenon and the varying optimal early stopping time for different models.

\paragraph{Test error for varying sample size.} In Figure~\ref{fig:error_mono}, we show the testing error as a function of sample size for training MNIST, CIFAR-10 and CIFAR-100. We observe that the test error decreases for increasing sample size.
\begin{figure}[h!]
\centering
\includegraphics[width=\linewidth]{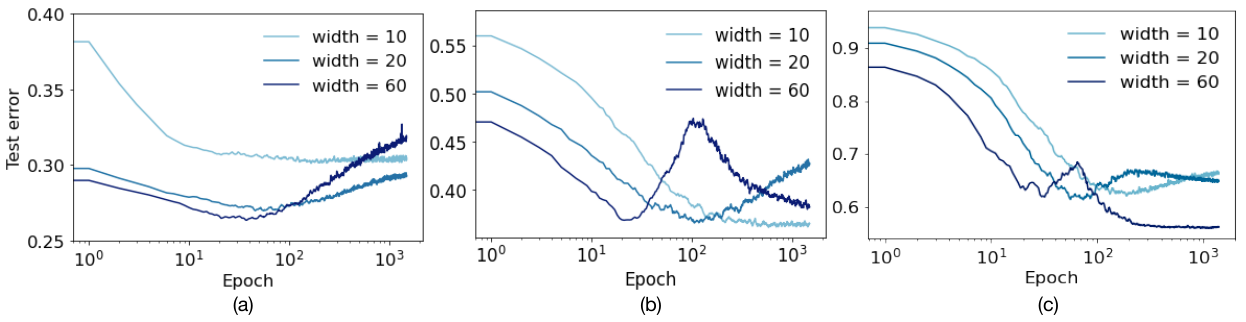}
\caption{Epoch-wise double descent and optimal early stopping time for (a) MNIST (b) CIFAR-10 (c) CIFAR-100.}
\label{fig:error_test}
\end{figure}

\begin{figure}[h!]
\centering
\includegraphics[width=\linewidth]{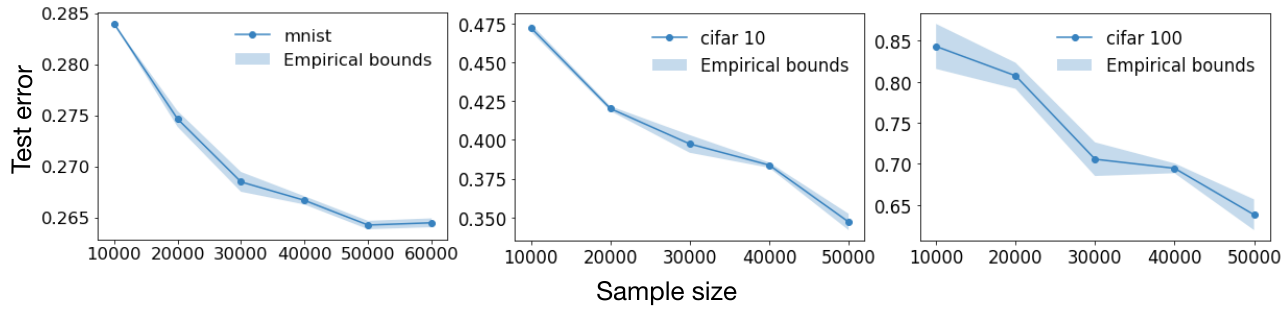}
\caption{Test error decreases with increasing sample size. }
\label{fig:error_mono}
\end{figure}

\begin{figure}[H]
\centering
\includegraphics[width=0.3\linewidth]{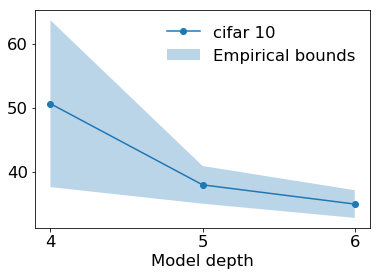}
\caption{Optimal Early Stopping time decreases with increasing depth.}
\label{fig:opt_depth}
\end{figure}
\paragraph{ResNet Depths} In Figure~\ref{fig:opt_depth}, we show that the optimal stopping time can also decrease with increasing ResNet depth. We train CIFAR-10 using ResNet. We add 20$\%$ of label noise to all datasets.
We use a family of ResNet networks with ResNet blocks of widths $=20$. 
When model depth $=4$, the ResNet block widths are $[20, 40, 80, 160]$; when model depth $=5$, the ResNet block widths are $[20, 40, 80, 160, 160]$; when model depth $=6$, the ResNet block widths are $[20, 40, 80, 160, 160, 160]$. The number of parameters grows up to 300M when depth $=6$. Inputs are normalized to $[-1, 1]$ with standard data-augmentation. We use stochastic gradient descent with cross-entropy loss, learning rate $\eta=0.1$, and minibatch size $B=512$. We choose the optimal stopping time similarly as in Section \ref{sec:exp_deep}. Increasing the depth of ResNet would similarly enlarge the network parameter size, resulting an observable decrease of optimal early stopping time. 

\begin{figure}[h!]

\begin{minipage}{0.45\textwidth}

  \includegraphics[width=0.8\linewidth]{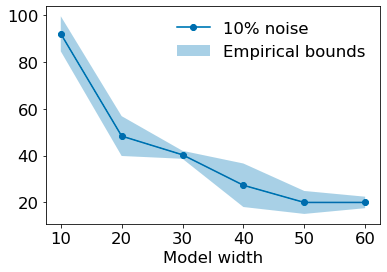}
  \caption{Optimal stopping time vs. model width with 10$\%$ label noise.}
    \label{fig:ln10}
\end{minipage}
\hfill
\begin{minipage}{.45\textwidth}

  \includegraphics[width=0.8\linewidth]{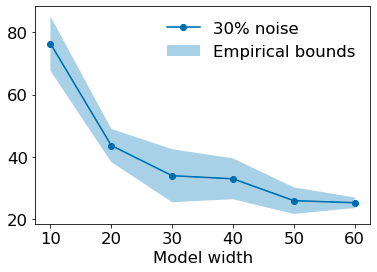}
  \caption{Optimal stopping time vs. model width with 30$\%$ label noise.}
  \label{fig:ln30}
\end{minipage}

\end{figure}
\paragraph{Label Noise Levels} In Figure~\ref{fig:ln10} and Figure~\ref{fig:ln30}, we show that the optimal stopping time decreases with increasing model width holds with 10$\%$ and 30 $\%$ label noise
when training on CIFAR-10 using ResNet. The variance of each independent experimental trail slightly breaks the monotonicity, yet the overall trend is clear.
We use a family of ResNet networks with convolutional layers of widths $[d, 2d, 4d, 8d]$ for different layer, and we set ResNet width from 10 to 60. 
Inputs are normalized to $[-1, 1]$ with standard data-augmentation. We use stochastic gradient descent with cross-entropy loss, learning rate $\eta=0.1$. and minibatch size $B=512$. We choose the optimal stopping time similarly as in Section \ref{sec:exp_deep}.

\end{document}